\documentclass[10pt,twocolumn,letterpaper]{article}

\usepackage{cvpr}      % To produce the REVIEW version
\usepackage{graphicx}
\usepackage{amsmath}
\usepackage{amssymb}
\usepackage{booktabs}
\usepackage{amsmath,amsfonts,amsthm}
\usepackage{amsmath,bm}
\usepackage{amssymb}
\usepackage{graphicx}
\usepackage[ruled,vlined]{algorithm2e}
\usepackage{enumitem}
\usepackage{caption}
\usepackage{multirow}
\usepackage{booktabs}

\usepackage[pagebackref,breaklinks,colorlinks]{hyperref}

\newtheorem{theorem}{Theorem}
\newtheorem{lemma}{Lemma}

\DeclareMathOperator*{\argmin}{argmin}

\newcommand{\w}{\bm{w}}
\newcommand{\x}{\bm{x}}

\newcommand{\comment}[1]{}

\newcommand{\cB}{\mathcal{B}}

\newcommand{\cD}{\mathcal{D}}

\newcommand{\cL}{\mathcal{L}}
\newcommand{\cS}{\mathcal{S}}

\newcommand{\cN}{\mathcal{N}}

\newcommand{\bbR}{\mathbb{R}}
\newcommand{\bbE}{\mathbb{E}}

\newcommand{\bbI}{\mathbb{I}}

\newcommand{\adv}{\textsf{adv}}
\newcommand{\rob}{\textsf{rob}}
\newcommand{\nat}{\textsf{nat}}
\newcommand{\batch}{\textsf{batch}}
\newcommand{\sign}{\textrm{sign}}

\usepackage[capitalize]{cleveref}
\crefname{section}{Sec.}{Secs.}
\Crefname{section}{Section}{Sections}
\Crefname{table}{Table}{Tables}
\crefname{table}{Tab.}{Tabs.}

\def\cvprPaperID{5142} % *** Enter the CVPR Paper ID here
\def\confName{CVPR}
\def\confYear{2023}

\begin{document}

\title{Towards Robust Dataset Learning}
\author{Yihan Wu\\
	University of Pittsburgh\\
	% PA, United States\\
	% {\tt\small yiw154@pitt.edu}
	\and
	Xinda Li\\
	University of Waterloo\\
	% ON, Canada\\
	\and
	Florian Kerschbaum\\
	University of Waterloo\\
	% ON, Canada\\
	\and
	Heng Huang\\
	University of Pittsburgh\\
	% PA, United States\\
	\and
	Hongyang Zhang\\
	University of Waterloo\\
	% ON, Canada\\
}
\maketitle

%%%%%%%%% ABSTRACT
\begin{abstract}
Adversarial training has been actively studied in recent computer vision research to improve the robustness of models. However, due to the huge computational cost of generating adversarial samples, adversarial training methods are often slow. 
In this paper, we study the problem of learning a robust dataset such that any classifier naturally trained on the dataset is adversarially robust.
Such a dataset benefits the downstream tasks as the natural training is much faster than adversarial training, and demonstrates that the desired property of robustness is transferable between models and data. In this work, we propose a principled tri-level optimization to formulate the robust dataset learning problem. We show that, under an abstraction model that characterizes robust vs. non-robust features, the proposed method provably learns a robust dataset. Extensive experiments on benchmark datasets demonstrate the effectiveness of our new algorithm with different network initializations and architectures.
\end{abstract}

%%%%%%%%% BODY TEXT
\section{Introduction}
\vspace{-5pt}
Deep learning models are vulnerable to adversarial examples \cite{szegedy2013intriguing,biggio2013evasion}: an adversary can arbitrarily manipulate the prediction results of deep neural networks with slight perturbations to the data. Many defense approaches, including heuristic defenses \cite{papernotDistillationDefenseAdversarial2016,xie2017mitigating,kannanAdversarialLogitPairing2018,liao2018defense,carmon2019unlabeled,mustafa2019adversarial,zhang2019theoretically,Zhang2020AttacksWD,Wu2020AdversarialWP,dong2020adversarial, tramer2020adaptive}
and certified defenses \cite{raghunathan2018certified,wong2018provable,wong2018scaling,singh2018fast,xiao2018training,gowal2018effectiveness,lecuyer2019certified,croce2019provable,li2019certified,cohen2019certified,zhang2020black,xu2020automatic,zhai2020macer,balunovic2020adversarial,zhang2021towards}, have been developed to protect deep learning models from these threats. The focus of this paper is on integrating the property of adversarial robustness into a dataset, such that a robust model (against small perturbation to the original test data) can be obtained through \emph{natural training} on the learned dataset.

There are several reasons of studying this problem. 1) Discovered by \cite{ilyas2019adversarial} in their seminal work, the desirable property of adversarial robustness is transferable between models and data. We propose a \emph{principled} approach of robust feature extraction with theoretical guarantees and improved empirical performance. 2) Expensive computational cost of most existing defenses hinders their applicability to the scenarios of limited computational resources. Although the task of learning robust dataset might itself be time-consuming, once the \emph{one-time} task has been outsourced to Machine Learning as a Service (MLaaS), 
one can benefit from the robust dataset for fast training of their own customized robust models, as natural training only requires light computational cost. 3) Distributing a robust dataset is more flexible than distributing a robust model. This is because loading a robust model requires extensive compatibility among deep learning framework (e.g., PyTorch, TensorFlow, MXNet, Keras, etc.), network architecture, and checkpoint. On the other hand, distributing a robust dataset allows everyone to train a network with their preferred architecture and deep learning framework for downstream tasks. Moreover, a robust dataset can be small, e.g., of size only 10\% of the original dataset, making it easy to transmit.

% Distilling knowledge from dataset 

However, there are only few works on robust dataset learning. A related but orthogonal research topic is dataset distillation \cite{wang2018dataset,cazenavette2022dataset}, which aims at reducing the scale of dataset by distilling knowledge from a large dataset to a small dataset. Despite dataset distillation has become a popular research topic in machine learning with
various applications \cite{bohdal2020flexible,nguyen2020dataset,sucholutsky2021soft,zhao2021dataset,zhao2021dataset2,nguyen2021dataset,cazenavette2022dataset}, how to learn a \emph{robust} dataset is less explored. To our best knowledge, the only attempt on building a robust dataset is by \cite{ilyas2019adversarial} on robust feature extraction. Thus, more theoretical and empirical results are desired for an in-depth understanding of robust dataset learning.

The idea behind our robust dataset learning is to represent a classifier as a function of a dataset, so that one can treat the dataset as a learnable parameter of the classifier. Throughout the paper, we name such a classifier the \emph{data-parameterized classifier}. We formulate robust dataset learning as a min-max, tri-level optimization problem. We theoretically show the efficiency of our algorithm, and empirically verify our robust dataset learning algorithm on MNIST, CIFAR10, and TinyImageNet datasets.
% In our work, we formulate robust dataset learning by the following min-max optimization problem
% \begin{equation}
% \label{equ: property-driven dataset learning}
% \begin{split}
% &\min_{\{x_i^\rob\}_{i=1}^{n_2}} \frac{1}{n_1}\sum_{i=1}^{n_1}\max_{x_i^\adv\in \mathcal{B}(x_i,\epsilon)} \mathcal{L}(f_{\theta(\{x_i^\rob\}_{i=1}^{n_2})}(x_i^\adv), y_i),\\
% \quad\text{s.t.}\quad &\theta(\{x_i^\rob\}_{i=1}^{n_2})=\argmin_\theta \frac{1}{n_2}\sum_{i=1}^{n_2}\mathcal{L}(f_\theta(x_i^\rob), y_i),
% \end{split}
% \end{equation}
% where $f_\theta$ stands for a neural network parameterized by $\theta$,  $\{(x_i,y_i)\}_{i=1}^{n_1}$ are the natural training data pairs, $\mathcal{L}$ is a loss function such as the cross entropy loss or hinge loss, and $\{(x_i^\rob,y_i)\}_{i=1}^{n_2}$ is the dataset to be learned such that the neural network fine-tuned naturally on the dataset is adversarially robust.
\begin{figure*}[t]
    \centering
    \includegraphics[height=6.5cm]{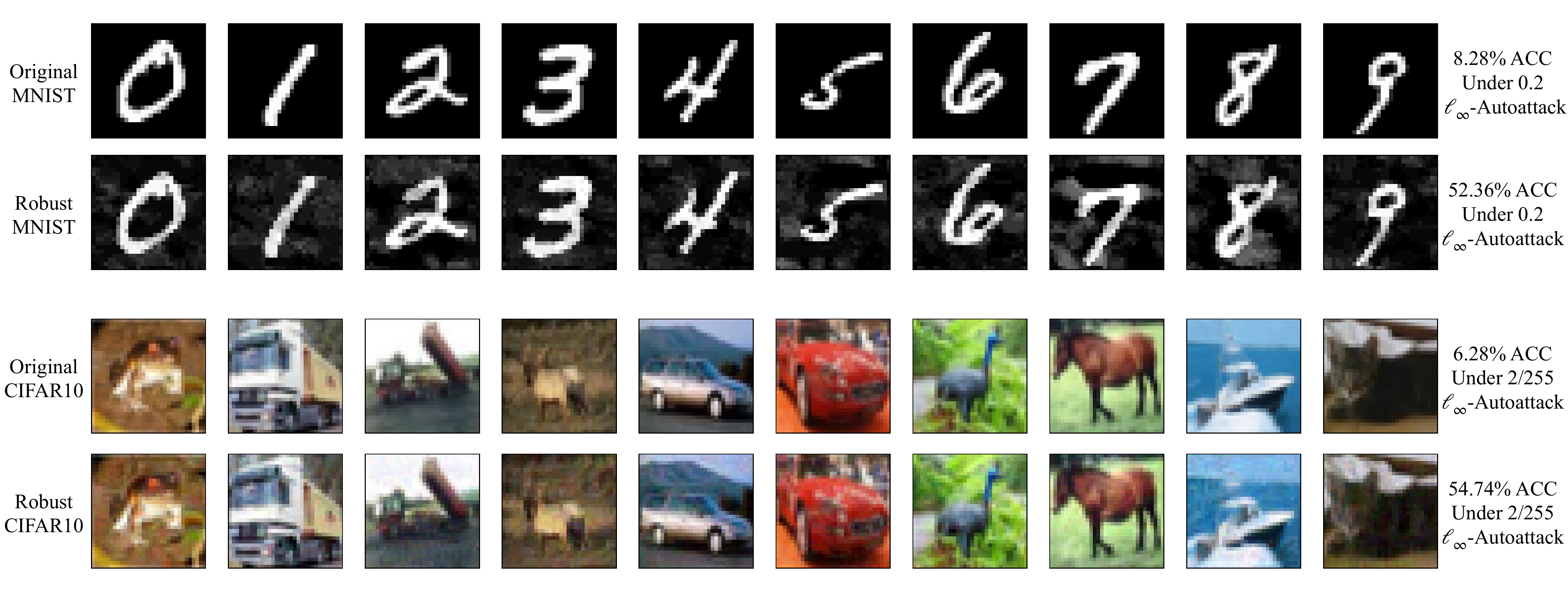}
    \vspace{-0.3cm}
    \caption{Illustration of original and robust images. For MNIST and CIFAR10, the first row represents the original images, while the second row represents the robust dataset generated by our algorithm. The rightmost column shows the robust accuracy.}
    \label{fig:compare}
    \vspace{-10pt}
\end{figure*}

\medskip
\noindent\textbf{Summary of contributions.}
Our work tackles the problem of robust dataset learning and advances the area in the following ways.

\begin{itemize}

\item Algorithmically, we formulate robust dataset learning via a tri-level optimization problem, where we parameterize the model weights by data, find adversarial examples of the model, and optimize adversarial loss over the data. This learning objective encourages the algorithm to maximize both clean and robust accuracy of the data-parameterized classifier.

\item  Theoretically, we investigate this tri-level optimization problem under an abstraction model that characterizes robust vs. non-robust features~\cite{tsipras2018robustness}, where the objective is to find a dataset that minimizes robust error on the data-parameterized classifier. We show that while the classifier naturally trained on clean dataset is non-robust, our data-parameterized classifier (trained on the robust dataset) is provably robust.

\item  Experimentally, we evaluate the clean and robust accuracy of our algorithm on MNIST, CIFAR10, and TinyImageNet. We consider baselines for robust dataset learning, which use datasets generated through adversarial attacks or robust feature extraction \cite{ilyas2019adversarial}. We show that our algorithm outperforms the baselines by a large margin. For example, on the CIFAR10 dataset with 0.25 $\ell_2$ threat model, our method achieves robust accuracy as high as 59.52\% under AutoAttack, beating the state-of-the-art 48.20\% in the same setting by a large margin.

\end{itemize}
\section{Related Works}
\noindent\textbf{Dataset distillation.} The ultimate goal of dataset distillation is to reduce training dataset size by distilling knowledge from the data. \cite{wang2018dataset}~proposed the first dataset distillation algorithm, which expressed the model parameters with the distilled images and optimized the images using gradient descent method. The subsequent works significantly improved the results by various strategies, such as learning soft labels \cite{sucholutsky2021soft}, strengthening learning signal through gradient matching \cite{zhao2021dataset}, adopting differentiable Siamese augmentation \cite{zhao2021dataset2}, optimizing with the neural tangent kernel under infinite-width limit \cite{nguyen2021dataset}, and matching training trajectories \cite{cazenavette2022dataset}.  Dataset distillation has been applied to many machine learning fields, including  federated learning \cite{sucholutsky2020secdd,zhou2020distilled}, privacy-preserving ML \cite{li2020soft}, and neural architecture search \cite{zhao2021dataset2}. 

\medskip
\noindent\textbf{Robust feature extraction.} \cite{ilyas2019adversarial} studied the existence and pervasiveness of adversarial examples. They theoretically demonstrated that adversarial examples are related to the presence of non-robust features, which are in fact highly predictive to neural networks, but brittle and imperceptible to humans. They proposed an empirical algorithm to separate the robust and non-robust features in the data to verify their theoretical results. A by-product of their algorithm is to obtain a robust model through natural training on the robust features. However, they did not provide a principled approach for robust dataset learning.

% After capturing these features within a theoretical framework, we establish their widespread existence in standard datasets. Finally, we present a simple setting where we can rigorously tie the phenomena we observe in practice to a misalignment between the (human-specified) notion of robustness and the inherent geometry of the data.

\section{Learning Robust Dataset: A Principled Approach}
% Motivated by our theoretical analysis in \autoref{sec:glimpse on SVM},
In this section, we propose a principled optimization formulation for robust dataset learning.

\medskip
\noindent\textbf{Problem settings.} There are two essential settings for the robust dataset learning problem: 1) construct a dataset; 2) obtain a robust model with only natural training on the constructed dataset.

\medskip
\noindent\textbf{Goal.}
 Given a original dataset $X^\nat$, The optimization objective is to find an optimal robust dataset $X^\rob$ such that the neural network that is naturally trained on $X^\rob$ is robust against adversarial perturbation to the clean test data, where the test data follows the same distribution as $X^\nat$.

\medskip
\noindent\textbf{Optimization.}
 The idea behind our method is to represent a classifier as a function of a dataset and to find the optimal dataset such that the classifier is robust against adversarial perturbations. With this idea, we formulate robust dataset learning as a tri-level optimization problem. Denote by $X^\nat:=\{(x_i,y_i)\}_{i=1}^{n}$ the original training data pairs, and by $X^\rob :=\{(x_i^\rob,y_i)\}_{i=1}^{n}$ the robust dataset to be learned. Notice we only optimize the data points $\{x_i^\rob\}_{i=1}^{n}$ and keep the labels $\{y_i\}_{i=1}^{n}$ unchanged.
% Following the intuition of problem \autoref{eq:property-driven dataset learning svm},
Step 1). For a given loss $\mathcal{L}$, in the first level we create a data-parameterized classifier $f_{\theta(X^\rob)}$ through minimizing the loss on $X^\rob$, which is initialized by $X^\nat$ and updated by gradient descent; Step 2). In the second level, we calculate the adversarial samples $X^\adv:=\{(x_i^\adv,y_i)\}_{i=1}^{n}$ of $X^\nat$ by attacking $f_{\theta(X^\rob)}$; Step 3). In the third level, we search for the optimal $X^\rob$ which minimizes the loss of $f_{\theta(X^\rob)}$ on $X^\adv$. 
With the above steps, our optimization problem is given by
% \hy{explain how the following optimization relates to \autoref{eq:property-driven dataset learning svm} in the SVM case}
% \begin{equation}
% \label{equ: property-driven dataset learning}
% \begin{split}
% \min_{\widetilde{X}} \frac{1}{n}\sum_{i=1}^{n} \ell(x_i, \theta(\widetilde{X})),
% \ \text{s.t.}\  \theta(\widetilde{X})=\argmin_{\theta} \frac{1}{n}\sum_{i=1}^{n}\ell(\widetilde{x}_i, \theta),
% \end{split}
% \end{equation}
\begin{equation}
\label{equ: property-driven dataset learning}
\begin{split}
&\min_{X^\rob} \frac{1}{n}\sum_{i=1}^{n}\max_{x_i^\adv\in \mathcal{B}(x_i,\epsilon)}\hspace{-0.3cm} \mathcal{L}(f_{\theta(X^\rob)}(x_i^\adv), y_i),
\\ &\text{s.t.}\  \theta(X^\rob)=\argmin_{\theta} \frac{1}{n}\sum_{i=1}^{n}\mathcal{L}(f_\theta(x_i^\rob), y_i),
\end{split}
\end{equation}
% \begin{equation}
% \label{equ: property-driven dataset learning}
% \begin{split}
% \min_{\theta} \frac{1}{n}\sum_{i=1}^{m}\max_{x_i^\adv\in \mathcal{B}(\widetilde{x}_i,\epsilon)}\hspace{-0.3cm} \mathcal{L}(f_{\theta}(x_i^\adv), y_i),
% \end{split}
% \end{equation}
% \begin{equation}
% \label{equ: property-driven dataset learning}
% \begin{split}
% \min_{\{x_i^\rob\}_{i=1}^{n_2}} \hspace{-0.1cm}\frac{1}{n}\sum_{i=1}^{n}\hspace{-0.05cm}\max_{x_i^\adv\in \mathcal{B}(x_i,\epsilon)} \hspace{-0.3cm}\mathcal{L}(f_{\theta(\{x_i^\rob\}_{i=1}^{n_2})}(x_i^\adv), y_i),
% \ \text{s.t.}\  \theta(\{x_i^\rob\}_{i=1}^{n_2})\hspace{-0.1cm}=\hspace{-0.1cm}\argmin_{\theta} \hspace{-0.1cm}\frac{1}{n_2}\hspace{-0.05cm}\sum_{i=1}^{n_2}\mathcal{L}(f_\theta(x_i^\rob), y_i),
% \end{split}
% \end{equation}
where $f_\theta$ stands for a neural network parameterized by $\theta$, $\cB(x_i,\epsilon)$ stands for a ball centered at $x_i^{\nat}$ with radius $\epsilon$, and $\mathcal{L}$ is a loss function, e.g., the cross entropy loss or hinge loss.

\medskip
\noindent\textbf{Efficient algorithm.}
However, \autoref{equ: property-driven dataset learning} is hard to be solved as it is a tri-level optimization problem. Denote by $t$ the current learning epoch.
The main difficulty is to find the closed form of the parameterized weight $\theta(X^\rob)$ which minimizes the loss w.r.t. $\{(x_i^\rob,y_i)\}_{i=1}^{n}$. Because the non-linearity of neural network, the closed form of $\theta(X^\rob)$ is always intractable,
empirically we use tens of thousands of gradient descent  to approximate $\theta(X^\rob)$. However, a memory issue will occur if we store all the gradient information during training. Thus, we use one-step gradient update to estimate $\theta(X^\rob)$
% We consider the second-order Taylor expansion of the loss $\mathcal{L}(f_\theta(x_i^\rob), y_i)$ at a predefined $\theta=\theta_0$ (e.g., a random set of parameters or the parameters of a naturally trained neural network) for a small $t>0$:
% \begin{equation*}
% \begin{split}
% \mathcal{L}(f_\theta(x_i^\rob), y_i)&\approx \mathcal{L}(f_{\theta_0}(x_i^\rob), y_i)+\langle\theta-\theta_0,\nabla_\theta \mathcal{L}(f_\theta(x_i^\rob),y_i)|_{\theta=\theta_0}\rangle\\
% &+\frac{1}{2t}\|\theta-\theta_0\|_2^2,
% \end{split}
% \end{equation*}
\begin{equation}\label{eqn:update theta}
\begin{split}
\theta(X^\rob)&=\argmin_\theta \frac{1}{n}\sum_{i=1}^{n}\mathcal{L}(f_\theta(x_i^\rob), y_i)\\
&\approx \theta_{t-1}-\gamma\frac{1}{n}\sum_{i=1}^{n}\nabla_\theta\mathcal{L}(f_\theta(x_i^\rob),y_i)|_{\theta=\theta_{t-1}},
\end{split}
\end{equation}
where $\theta_{t-1}$ is the network weight from the previous epoch and $\gamma$ is the learning rate.
% Therefore, we have the following approximation of \autoref{equ: property-driven dataset learning}:
% \begin{equation}
% \label{equ: approximate property-driven dataset learning}
% \begin{split}
% &\min_{X^\rob} \frac{1}{n}\sum_{i=1}^{n}\max_{x_i^\adv\in \mathcal{B}(x_i,\epsilon)} \mathcal{L}(f_{\theta(X^\rob)}(x_i^\adv), y_i),
% \\ &\text{s.t.}\ \theta(X^\rob)=\theta_0-t\frac{1}{n}\sum_{i=1}^{n}\nabla_\theta\mathcal{L}(f_\theta(x_i^\rob),y_i)|_{\theta=\theta_0}.
% \end{split}
% \end{equation}
% \begin{equation}
% \label{equ: approximate property-driven dataset learning}
% \begin{split}
% &\min_{\{x_i^\rob\}_{i=1}^{n_2}} \frac{1}{n}\sum_{i=1}^{n}\max_{x_i^\adv\in \mathcal{B}(x_i,\epsilon)} \mathcal{L}(f_{\theta(\{x_i^\rob\}_{i=1}^{n_2})}(x_i^\adv), y_i),\\
% \quad\text{s.t.}\quad &\theta(\{x_i^\rob\}_{i=1}^{n_2})=\theta_0-t\frac{1}{n_2}\sum_{i=1}^{n_2}\nabla_\theta\mathcal{L}(f_\theta(x_i^\rob),y_i)|_{\theta=\theta_0}.
% \end{split}
% \end{equation}
\begin{algorithm}[t]
	\SetAlgoLined
	\KwIn{original training set $X^\nat$; number of training epochs $T$;  classifier $f$ (with weight $\theta$), initialized by $\theta_0$; learning rate $\gamma$ of classifier; learning rate $\beta$ of robust dataset; PGD steps for generating adversarial example $s$; PGD steps size $\alpha$; PGD attack budgets $\epsilon$.}
	
	initialize classifier weights $\theta$ with $\theta_0$, initialize $X^\rob$ with $X^\nat$;\\
	\For{1:T}{
	\For{mini-batches $(b^\nat,y^\batch)\subseteq X^\nat$ and $(b^\rob,y^\batch)\subseteq X^{\rob}$}{
% 	\blue{// update classifier with $s_2$}\\
	\textbf{step 1.} update classifier with robust data via \autoref{eqn:update theta}:\\ $\theta(X^\rob)\leftarrow \theta(X^\rob) -\gamma \frac{1}{|b^\rob|}\sum_{(x,y)\in (b^\rob,y^\batch)}\nabla_\theta\mathcal{L}(f_\theta(x),y)$;\\
	\textbf{step 2.} calculate adversarial examples from original training set via PGD attack: \\
	for each $(x^\nat,y)\in (b^\nat,y^\batch)$, initialize $x^\adv$ with $x^\nat$;\\
	\For{$1:s$}{
	generate perturbation through FSGM and clip it within a ball centered at $x^\nat$ with radius $\epsilon$:
	$x^\adv\leftarrow \text{Clip}_{\mathcal{B}(x^\nat,\epsilon)}(x^\adv + \alpha \sign(\nabla_{x^\adv}\mathcal{L}(f_\theta(x^\adv),y)))$;\\
	}
	\textbf{step 3.} for each $x^\rob\in b^\rob$, update robust data by minimizing robust error:
	$x^\rob\leftarrow x^\rob-\beta \sign(\frac{1}{|b^\rob|}\sum_{(x^\adv,y)}\nabla_{x^\rob}\mathcal{L}(f_\theta(x^\adv),y))$;\\
	}
	}

	\textbf{return} robust dataset $X^{\text{rob}}$.
	\caption{Robust dataset learning.}
	\label{alg:1}
\end{algorithm}
To solve the inner maximization problem, we apply PGD-attack \cite{madry2017towards} via repeatedly using
	$$x_i^\adv\leftarrow \text{Clip}_{\mathcal{B}(x_i^\nat,\epsilon)}(  x_i^\adv + \alpha \sign(\nabla_{x_i^\adv}\mathcal{L}(f_\theta(x_i^\adv),y_i))),$$
where $\epsilon$ is the attack budget, $\cB(x_i^\nat,\epsilon)$ stands for a ball centered at $x_i^{\nat}$ with radius $\epsilon$, $\text{Clip}_{\mathcal{B}(x_i^\nat,\epsilon)}$ is a clip function that restricts adversarial samples to $\cB(x_i^\nat,\epsilon)$, and $\alpha$ is the step size of the PGD-attack. We use gradient descent to optimize the robust dataset:
$$ x_i^\rob\leftarrow x_i^\rob-\beta \sign(\frac{1}{n}\sum_{i=1}^{n} \nabla_{x_i^\rob}\mathcal{L}(f_{\theta(\{x_i^\rob\}_{i=1}^{n})}(x_i^\adv),y_i)),$$
where $\beta$ is the learning rate for the robust data. Notice that since we focus on the image classification tasks, we apply fast sign gradient method to modify the data. This method can be replaced by other gradient descent methods for a specific task. The details of our robust learning algorithm is shown in \autoref{alg:1}.

\medskip
\noindent\textbf{Comparing to dataset distillation.} The dataset distillation problem is usually formulated by a bi-level optimization problem. For example, the optimization problem of \cite{wang2018dataset} is given by
\begin{equation}
\begin{split}
&\min_{X^*} \frac{1}{n}\sum_{i=1}^{n} \mathcal{L}(f_{\theta(X^*)}(x_i), y_i),
\\ &\text{s.t.}\  \theta(X^*)=\theta_0 -\gamma\frac{1}{n}\sum_{i=1}^{k}\nabla_\theta\mathcal{L}(f_\theta(x_i^*),y_i^*)|_{\theta=\theta_{0}},
\end{split}
\end{equation}
where $X^\nat = \{(x_i,y_i\}_{i=1}^n$ is the natural dataset, and $X^* = \{(x_i^*,y_i^*\}_{i=1}^k$ is the distilled dataset which has significant small volume, i.e., $k\ll n$. There are also different formulations of the dataset distillation problem, e.g., matching training trajectory/gradient, but they still share the similar bi-level optimization framework. Our optimization problem is tri-level with large distilled (robust) dataset ($k=n$), which is harder to solve and requires more carefully algorithm design.

\medskip
\noindent\textbf{Difference between our work and \cite{ilyas2019adversarial}.} \cite{ilyas2019adversarial} formulated the robust feature extraction problem by finding $\{(x_i^\rob,y_i)\}_{i=1}^{n}$ that minimizes $||g(x_i^\rob)-g(x_i)||_2$, where $g$ is a robust pre-trained representation model trained by adversarial training, and a proper initialization of $x_i^\rob$ is involved to avoid converging to a trivial solution $x_i$. As we will see in Sections \ref{sec:glimpse on SVM} and \ref{section: experiments}, our new formulation \autoref{equ: property-driven dataset learning} not only provably extracts the robust feature under a simple abstraction model, but also enjoys an improved empirical performance on many experiment settings compared with \cite{ilyas2019adversarial}.

% \end{proof}
\section{Theoretical Analysis}\label{sec:glimpse on SVM}

In this section, we present a fairly simple theoretical model to analyze the above-mentioned robust dataset learning problem. Our analysis is structured as follows: In \autoref{sec:settings}, we introduce the problem settings of the data distribution, classifier, and the optimization objective of the robust dataset learning problem; In \autoref{sec:natural classifier}, we prove that the optimal classifier trained on clean dataset can be non-robust; In \autoref{sec:robust dataset}, we demonstrate that our optimization objective leads to a robust dataset. All proof details can be found in \autoref{sec:missing proof}.

\subsection{Problem settings}\label{sec:settings}
\medskip
\noindent\textbf{Setting.}
We consider the data distribution commonly used in prior works \cite{tsipras2018robustness, ilyas2019adversarial}, where the instance $\x=(x_1,...,x_{d+1})\sim\cD$ and the label $y$ follow: 
%  \hy{give the credit to Madry, justify why the distribution is interesting}
\begin{equation}\label{eq:distribution}
\begin{split}
    &y\sim \textrm{Uniform}\{-1,1\},\quad x_1\sim\left\{\begin{aligned}
         y,\quad &\textrm{with prob. } p;\\
        - y,\quad &\textrm{with prob. } 1-p,
    \end{aligned}\right.\\
    &x_i\sim \cN(\mu y,1),\ i=2,3,...,d+1,
\end{split}
\end{equation}
where $\cN(\mu,\sigma^2)$ is a Gaussian distribution with mean $\mu=o(1)$ and variance $\sigma^2$. It consists of strongly-correlated (to the label) feature $x_1$ and weakly-correlated features $x_2,...,x_{d+1}$ (as $\mu$ is selected small enough).  The strongly-correlated feature is robust against $\Theta(1)$  perturbations but the weakly-correlated features are not. Thus, a naturally trained classifier on this distribution is non-robust as it puts heavy weights on non-robust features. Besides, it is easy to achieve high natural accuracy on this data distribution. For example, we can set $\mu=\Theta(1/\sqrt{d})$ such that a simple classifier, e.g., $\sign(x_2)$, can achieve at least 99\% natural accuracy. The probability $p$ quantifies the correlation between the feature $x_1$ and the label $y$. We can set $p$ to be moderately large, i.e., $p=0.97$. We note that the same data distribution was used to demonstrate the trade-off between robustness and accuracy~\cite{tsipras2018robustness} and to provide a clean abstraction of robust and non-robust features~\cite{ilyas2019adversarial}. In this paper, we use the same distribution to show a separation between natural training on the clean dataset and on the robust dataset returned by our framework.

We model the natural training by a soft SVM loss (a.k.a. the hinge loss) of a linear classifier:
\begin{equation}\label{eq:svm}
\min_{\w}\cL(\w;\cD):=\bbE_{(\x\sim\cD,y)}[\max\{0,1-y\w^T\x\}]+\lambda||\w||_2^2,
\end{equation}
where $\w = (w_1,...,w_{d+1})$ is the weight vector and $\lambda>0$ is a regularization parameter. For a given distribution $\cD'$, the data-parameterized classifier (w.r.t. $\cD'$) is given by $\sign(\w_{\cD'}^T\x)$, where $\w_{\cD'} := \argmin_{\w}\cL(\w;\cD')$ is the optimal weight of the SVM w.r.t. $\cD'$.

\medskip
\noindent\textbf{Goal.}
Our goal is to create a robust data distribution\footnote{We study the popular error in this section, where the problem of robust dataset learning reduces to the problem of robust data distribution learning.} $\cD'$, such that $\w_{\cD'}$ is robust. We formulate this problem via our proposed tri-level optimization framework \autoref{equ: property-driven dataset learning}, where the algorithm is supposed to find $\cD'$ that minimizes the adversarial loss w.r.t. the weight $\w_{\cD'}$ and the worst-case perturbation $\delta$.  In particular, we define our robust dataset learning problem as

% we propose a tri-level optimization framework to formulate this problem, where the algorithm is supposed to find $\cD'$ that minimizes the adversarial loss w.r.t. the weight $\w_{\cD'}$ and the worst-case perturbation $\delta$.  In particular, we define our robust dataset learning problem as
% \begin{definition}[property-driven dataset learning problem]
\begin{equation}\label{eq:property-driven dataset learning svm}
\begin{split}
    &\min_{\cD'} \bbE_{(\x,y)}[\max_{||\bm\delta||_\infty\leq \epsilon}\hspace{-0.1cm}\max\{0,1\hspace{-0.1cm}-\hspace{-0.1cm}y\w^{T}_{\cD'}(\x+\bm\delta)\}]\hspace{-0.1cm}+\hspace{-0.1cm}\lambda||\w_{\cD'}||_2^2, \\
    &\text{s.t.}\ \  \w_{\cD'} := \argmin_{\w}\cL(\w;\cD'),
    % = \min_{\cD'}\bbE_{(\x,y)}[\max\{0,1-y\w^{T}_{\cD'}\x+\epsilon||\w_{\cD'}||_1\}]
\end{split}
\end{equation}
where the inner maximization $$ \bbE_{(\x,y)}[\max_{||\bm\delta||_\infty\leq \epsilon}\max\{0,1-y\w^{T}_{\cD'}(\x+\bm\delta)\}]+\lambda||\w_{\cD'}||_2^2$$ is the adversarial loss that applies $\ell_\infty$ perturbation with budget $\epsilon$ to attack $\w_{\cD'}$, and the outer minimization optimizes the adversarial loss w.r.t. $\cD'$. Intuitively, the optimal solution of this min-max problem implies a robust data distribution.
% \end{definition}

\subsection{Natural training on the clean dataset is non-robust}\label{sec:natural classifier}
As a comparison, we begin by showing that the optimal classifier of \autoref{eq:svm} is non-robust. 
 
We consider the robustness of classifiers under $\ell_\infty$ adversarial perturbations with attack budget $\epsilon$, which means that an adversary can modify each feature by at most a value of $\pm\epsilon$. Note that the first feature of $\x$ (\autoref{eq:distribution}) is strongly correlated with the label, and the rest $d$ features are only weakly correlated with the label. We prove that both strongly and weakly-correlated features contribute to the prediction in the optimal classifier, while the effect of the weakly-correlated features dominates the strongly-correlated one, i.e., the weight of SVM on the weakly-correlated features is larger than the weight on the strongly-correlated feature. Under $\ell_\infty$-perturbations with  $\epsilon=\Theta(1/\sqrt{d})$, the positive effect of weakly-correlated features will be overridden by the perturbation, i.e., the weakly-correlated features hurt the prediction under attacks. For example, if $\epsilon=2\mu$, the weakly-correlated features will be shifted to be anti-correlated features by the adversary, i.e., ($\cN(\mu y,1)\to\cN(-\mu y,1)$). As the weakly-correlated features dominate the prediction of the optimal classifier, the classifier will predict opposite labels under such perturbations. In the following lemma, we formally state the above discussion.

\begin{lemma}\label{thm:natural train}
If $\mu\geq \frac{4}{\sqrt{d}}$ and $p\leq 0.975$, the optimal classifier $\w^*=(w_1^*,...,w_{d+1}^*)$ of \autoref{eq:svm} achieves more than 99\% natural accuracy but less than 0.2\% robust accuracy with $\ell_\infty$-perturbation of size $\epsilon\geq2\mu$. %$\Theta(\frac{1}{\sqrt{d}})$.
\end{lemma}

This lemma is adapted from Theorem 2 of \cite{tsipras2018robustness}, which states that the optimal classifier naturally trained on $\x$ has low accuracy under $\Theta(1/\sqrt{d})$ $\ell_\infty$-attacks, which indicates achieving robustness on this dataset is non-trivial.

\subsection{Natural training on the robust dataset of our framework is robust}\label{sec:robust dataset}
In this part, we will show that the optimal dataset of our min-max optimization  \autoref{eq:property-driven dataset learning svm} can lead to a robust classifier against $\Theta(1)$ $\ell_{\infty}$-perturbations.

We start with calculating the strongest $\ell_\infty$ adversarial perturbations of the SVM classifier.
\begin{lemma}\label{lm:pertubation}
For arbitrary $\w$, the optimal $\bm\delta$ of the maximization problem $$ \max_{||\bm\delta||_\infty\leq \epsilon}\max\{0,1-y\w^{T}(\x+\bm\delta)\}$$
is given by $\bm\delta =- \epsilon \sign(y\w)$.
\end{lemma}

% \begin{proof}[Proof sketch]
This lemma provides a closed form of strongest $\ell_\infty$ adversarial perturbations of the SVM classifier. It's easy to verify this lemma via Holder's inequality $-y\w^{T}\bm\delta\leq ||\w||_1||\bm\delta||_\infty = \epsilon||\w||_1$. Thus $\bbE_{(\x,y)}[\max\{0,1-y\w^{T}(\x+\bm\delta)\}]\leq\bbE_{(\x,y)}[\max\{0,1-y\w^{T}\x+\epsilon||\w||_1\}]$. Taking $\bm\delta =- \epsilon \sign(y\w)$, we can reach this maximum.
% \end{proof}

According to \autoref{lm:pertubation}, the inner maximization problem has a closed form solution
\begin{equation*}
\begin{split}
    & \bbE_{(\x,y)}[\max_{||\bm\delta||_\infty\leq \epsilon}\max\{0,1-y\w^{T}_{\cD'}(\x+\bm\delta)\}] \\
    = & \bbE_{(\x,y)}[\max\{0,1-y\w^{T}_{\cD'}\x+\epsilon||\w_{\cD'}||_1\}].
    \end{split}
\end{equation*}

Thus, we only need to solve the minimization problem
\begin{equation}\label{eq:dataset learning opt}
    \min_{\cD'}\bbE_{(\x,y)}[\max\{0,1-y\w^{T}_{\cD'}\x+\epsilon||\w_{\cD'}||_1\}]+\lambda||\w_{\cD'}||_2^2.
    % = \min_{\cD'}\bbE_{(\x,y)}[\max\{0,1-y\w^{T}_{\cD'}\x+\epsilon||\w_{\cD'}||_1\}]
\end{equation}
However, it is computationally intractable to find the optimal distribution $\mathcal{D}'$ directly, as we cannot represent $\w_{\cD'}$ with $\cD'$ explicitly. Instead, we try to find the necessary and sufficient conditions of $\w_{\cD'}$ that minimizes \autoref{eq:dataset learning opt}. We show the existence of distributions such that the related $\w_{\cD'}$ satisfies the conditions.

% \begin{lemma}Assume the optimal solution of \autoref{eq:svm} corresponding to distribution $\x'\sim\cD'$ is $\w_{\cD'}$ i.e. $\w_{\cD'} = \argmin_{\w}\bbE_{(\x',y)}[\max\{0,1-y\w^T\x'\}]$
% \begin{equation}
%     \max_{||\bm\delta||_\infty\leq \epsilon} \bbE_{(\x,y)}[\max\{0,1-y\w^{T}_{\cD'}(\x+\bm\delta)\}] =  \bbE_{(\x,y)}[\max\{0,1-y\w^{T}_{\cD'}\x+\epsilon||\w_{\cD'}||_1\},
% \end{equation}
% one of the optimal $\bm\delta =- \epsilon \sign(y\w_{\cD'})$
% \end{lemma}

\begin{theorem}\label{lm:optweight}
% Our minmax optimization problem is
% \begin{equation}
%     \min_{\cD'}\max_{||\bm\delta||_\infty\leq \epsilon} \bbE_{(\x,y)}[\max\{0,1-y\w^{T}_{\cD'}(\x+\bm\delta)\}] = \min_{\cD'}\bbE_{(\x,y)}[\max\{0,1-y\w^{T}_{\cD'}\x+\epsilon||\w_{\cD'}||_1\}]
% \end{equation}

% the optimal $\alpha,\mu,\sigma$ should satisfy $\mu=0,\sigma=0$
Let
\begin{equation*}
\begin{split}
    \cD^*&= \argmin_{\cD'}\bbE_{(\x,y)}[\max\{0,1-y\w^{T}_{\cD'}\x+\epsilon||\w_{\cD'}||_1\}]\\
    &+\lambda||\w_{\cD'}||_2^2.
\end{split}
\end{equation*}
If $1>\epsilon\geq\mu$ , then $\w_{\cD^*} := (w_{\cD^*}^{(1)},...,w_{\cD^*}^{(d+1)})$ must satisfy $w_{\cD^*}^{(1)}>0$ and $w_{\cD^*}^{(2)}=w_{\cD^*}^{(3)}=...=w_{\cD^*}^{(d+1)}=0$.
\end{theorem}

This theorem states a necessary condition of the optimal dataset from \autoref{eq:property-driven dataset learning svm}. The robust dataset should suffice that the data-parameterized classifier (naturally trained on the dataset) is independent of the weak-correlated features $x_2,...,x_{d+1}$. Thus, the data-parameterized classifier should be more robust than vanilla classifier trained on original dataset.

\begin{theorem}\label{thm:robustdata}
The optimal SVM w.r.t.~the robust dataset $\cD^*$ has clean and robust accuracy $\ge p$  under $\ell_\infty$-perturbation (with budget less than 1) on the original dataset. An optimal solution of $\cD^*$ is given by
\begin{equation*}
\begin{split}
    &y\sim \textrm{Uniform}\{-1,1\},\quad x_1\sim\left\{\begin{aligned}
         y,\quad &\textrm{with prob. } p;\\
        - y,\quad &\textrm{with prob. } 1-p,
    \end{aligned}\right.\\
    &x_i=1,\ i=2,3,...,d+1.
\end{split}
\end{equation*}
\end{theorem}
% \begin{proof}

% \begin{proof}[Proof sketch.]
According to \autoref{lm:optweight}, the weight of the optimal SVM learned from $\cD^*$ should satisfy  $w_{\cD^*}^{(2)}=w_{\cD^*}^{(3)}=...=w_{\cD^*}^{(d+1)}=0$. Thus the clean and robust accuracy are only related to the first feature of $\x$. It is easy to see the natural accuracy is equal to the probability that $\sign(x_1)=y$, which is $p$. Besides, when the perturbation budget $\epsilon<1$, the adversary does not change the sign of $x_1$. Thus the robust accuracy is also $p$.
% \end{proof}

Compared to the original distribution $\cD$ (\autoref{eq:distribution}), the robust distribution $\cD^*$ keeps the strongly-correlated feature $x_1$ unchanged and modifies the weakly-correlated features $x_2,...,x_{d+1}$ to uncorrelated features (a constant). In this way the optimal SVM trained on the robust distribution  will not assign weights on the uncorrelated features, because they do not contribute to the predictions. Thus, the resulting classifier is relatively robust, as it depends only on the strongly-correlated feature.

% \end{proof}

\begin{table*}[t]
\caption{Experimental results of robust dataset learning on MNIST, CIFAR10 and TinyImageNet with Autoattack, where we naturally train classifiers on the datasets created by different methods. Numbers with $*$ refer to the experimental results with 1000-PGD attack reported by the original work.}
% \vspace{-3pt}
\label{tab:robust dataset}
\resizebox{1.0\textwidth}{!}{
\begin{tabular}{clcccc}
\bottomrule
\multirow{5}{*}{MNIST}        & Robust acc (\%) / Natural acc (\%)                     & 0.1 ($\ell_\infty$)                   & 0.2 ($\ell_\infty$)                   & 1.0 ($\ell_2$)                        & 2.0 ($\ell_2$)                        \\ \cline{2-6} 
                              & Natural dataset                              & 71.73/98.10                           & 8.28/98.10                            & 79.14/98.10                           & 21.28/98.10                           \\
                              & Adv. data of natural classifier              & 84.17/97.41                           & 19.40/94.71                           & 86.81/97.85                           & 26.58/95.80                           \\
                              & Adv. data of robust classifier               & 76.70/97.93                           & 11.79/97.50                           & 79.67/97.58                           & 24.21/97.83                           \\
                              & Robust dataset (ours)                        & \textbf{93.53}/98.69 & \textbf{52.36}/97.29 & \textbf{91.60}/98.76 & \textbf{48.40}/98.25 \\ \bottomrule
                              &                                              & \multicolumn{1}{l}{}                  & \multicolumn{1}{l}{}                  & \multicolumn{1}{l}{}                  & \multicolumn{1}{l}{}                  \\ \bottomrule
\multirow{6}{*}{CIFAR10}      & Robust acc (\%) / Natural acc (\%)                     & 2/255 ($\ell_\infty$)                 & 4/255 ($\ell_\infty$)                 & 0.25 ($\ell_2$)                       & 0.5 ($\ell_2$)                        \\ \cline{2-6} 
                              & Natural dataset                              & 6.28/93.23                            & 0.02/93.23                            & 9.85/93.23                            & 0.04/93.23                            \\
                              & Adv. data of natural classifier              & 48.21/84.66                           & 17.86/80.48                           & 47.05/81.83                           & 21.33/81.14                           \\
                              & Adv. data of robust classifier               & 10.51/86.06                           & 0.21/85.83                            & 11.68/88.70                           & 0.39/86.64                            \\
                              & Ilyas et al. \cite{ilyas2019adversarial} & 36.36/77.53                          & 14.56/78.61                        &  48.20*/85.40*                                & 21.85*/85.40*                               \\
                              & Robust dataset (ours)                        & \textbf{54.74}/87.19 & \textbf{26.79}/85.55 & \textbf{59.52}/86.59 & \textbf{27.35}/85.10 \\ \bottomrule
                              &                                              & \multicolumn{1}{l}{}                  & \multicolumn{1}{l}{}                  & \multicolumn{1}{l}{}                  & \multicolumn{1}{l}{}                  \\ \bottomrule
\multirow{4}{*}{TinyImageNet} & Robust acc (\%) / Natural acc (\%)                     & 2/255 ($\ell_\infty$)                 & 4/255 ($\ell_\infty$)                 & 0.25 ($\ell_2$)                       & 0.5 ($\ell_2$)                        \\ \cline{2-6} 
                              & Natural dataset                              & 0.34/70.94                            & 0.16/70.94                            & 4.55/70.94                            & 0.52/70.94                            \\
                              & Adv. data of natural classifier              & 12.96/65.22                           & 5.10/65.13                            & 29.97/65.98                           & 9.36/64.14                            \\
                              & Robust dataset (ours)                        & \textbf{25.43}/60.02                           & \textbf{18.42}/60.36                           & \textbf{39.55}/61.13                           & \textbf{25.48}/60.92                           \\ \bottomrule
\end{tabular}
}
\end{table*}

% \medskip
% \noindent\textbf{Extension to general data distributions.} 
\subsection{Extension to general data distributions.}
We now show that our theorems in \autoref{sec:robust dataset} hold for a more general distribution. Consider the case where the instance $\x$ and the label $y$ follow the distribution below:
\begin{equation*}
\begin{split}
    &y\sim \textrm{Uniform}\{-1,1\},\quad x_1\sim\left\{\begin{aligned}
         y,\quad &\textrm{with prob. } p;\\
        - y,\quad &\textrm{with prob. } 1-p,
    \end{aligned}\right.\\
    &x_i\sim \cD_i,\ i=2,3,...,d+1,
    \end{split}
\end{equation*}
where $\cD_i$ are symmetric distributions with mean $\mu_i\leq 1$. We prove that the parameterized SVM with the optimal robust dataset $\cD^*$ achieves at least $p$ clean and robust accuracy under $\ell_\infty$-perturbation (with budget less than 1). The details can be found in \autoref{sec:general dataset}.

\section{Experiments}
\label{section: experiments}

In this section, we conduct comprehensive experiments to demonstrate the effectiveness of our algorithm on MNIST \cite{lecun1998mnist}, CIFAR10 \cite{krizhevsky2009learning}, and TinyImageNet \cite{deng2009imagenet}. 

\subsection{Robustness}\label{sec:robustness}
In this part, we compare the performance of our (robust) data-parameterized model with models obtained from several baseline methods under $\ell_2$ and $\ell_\infty$ attacks. We use the state-of-the-art attack method---Autoattack \cite{croce2020reliable} for evaluating the adversarial robustness of models.

\medskip
\noindent\textbf{Baseline.} \cite{ilyas2019adversarial} is the only work related to robust dataset learning. We include this work as one of the baseline for CIFAR10.\footnote{\cite{ilyas2019adversarial} released a robust dataset on CIFAR10 but did not release the code for generating these images. See \url{https://github.com/MadryLab/constructed-datasets}. In this work, we use the $\ell_2$ robustness reported in their work for comparison, and evaluate $\ell_\infty$ robustness on the released CIFAR10 dataset.}  Besides, motivated by adversarial training, we create two other baselines. In adversarial training, we utilize adversarial examples to improve robustness, so we take the adversarial data generated from both natural (see Adv. data of natural classifier in \autoref{tab:robust dataset}) and robust pre-trained classifiers (see Adv. data of robust classifier in \autoref{tab:robust dataset}) as two baseline robust datasets. In order to make a fair comparison, we require all robust datasets to have the same size.

\medskip
\noindent\textbf{General settings.} In data pre-procession phase, we randomly cropped the image to 28$\times$28 for MNIST, 32$\times$32 for CIFAR10, and 64$\times$64 for TinyImageNet with 4 pixes padding. Then we apply random horizontal flip to the images and normalize them with mean 0.1307 and variance 0.3081. During training,  we use SGD \cite{bottou2010large} with learning rate 0.01, momentum 0.9, weight decay 5e-4, and cosine learning rate decay to fine tune the models. For the robust pre-train model we train TRADES \cite{zhang2019theoretically} with 0.2 $\ell_\infty$ perturbations for MNIST and 4/255 $\ell_\infty$ perturbations for CIFAR 10. 

\medskip
\noindent\textbf{Evaluation.} During the evaluation, we fine-tune a natural classifier with the given (robust) dataset and evaluate the model with adversarial attacks. For MNIST and CIFAR-10, we use Autoattack \cite{croce2020reliable} with various budgets to evaluate the robustness of the models on the test set.
Since Autoattack is computationally expensive on TinyImageNet, we evaluate our algorithm with the same set of budgets using PGD-10 attack instead.\footnote{Additional results of TinyImageNet using Autoattack can be found in \autoref{sec:app_exp}.}

\begin{figure*}[t]
		\begin{minipage}[t]{.24\linewidth}
			\centering
			\includegraphics[height=3.4cm]{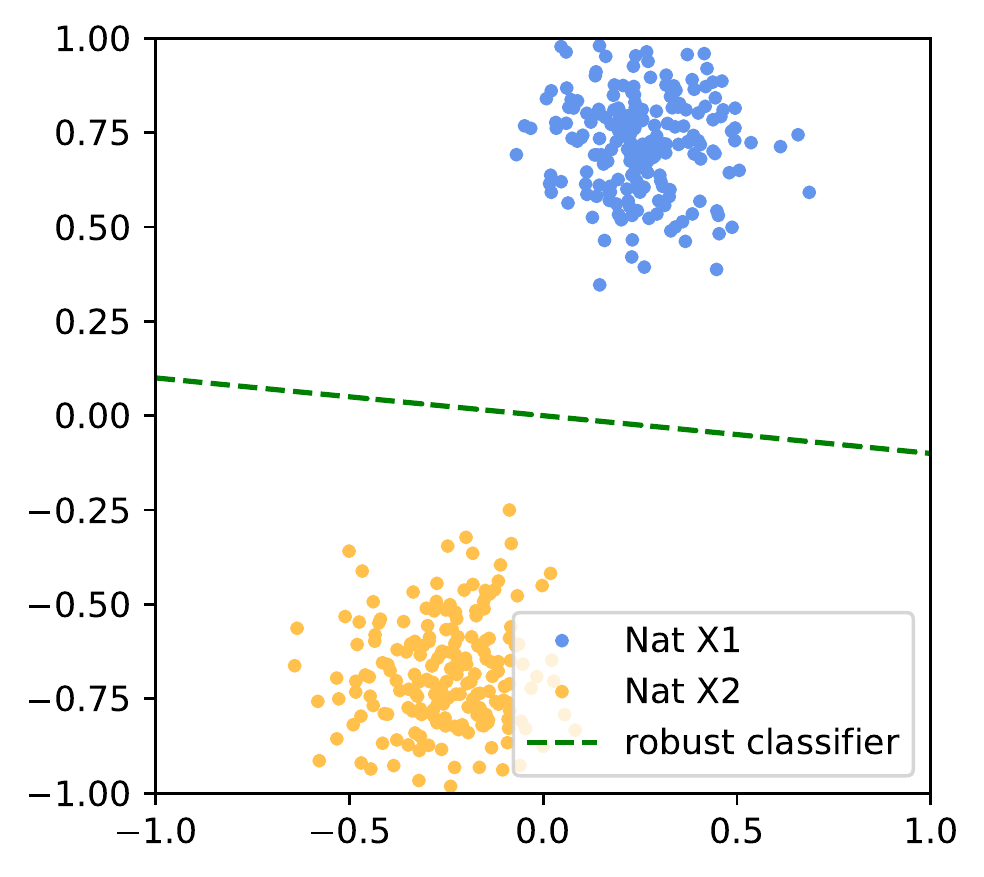}
			\vspace{-6pt}
		\end{minipage}
		\begin{minipage}[t]{.24\linewidth}
			\centering
			\includegraphics[height=3.4cm]{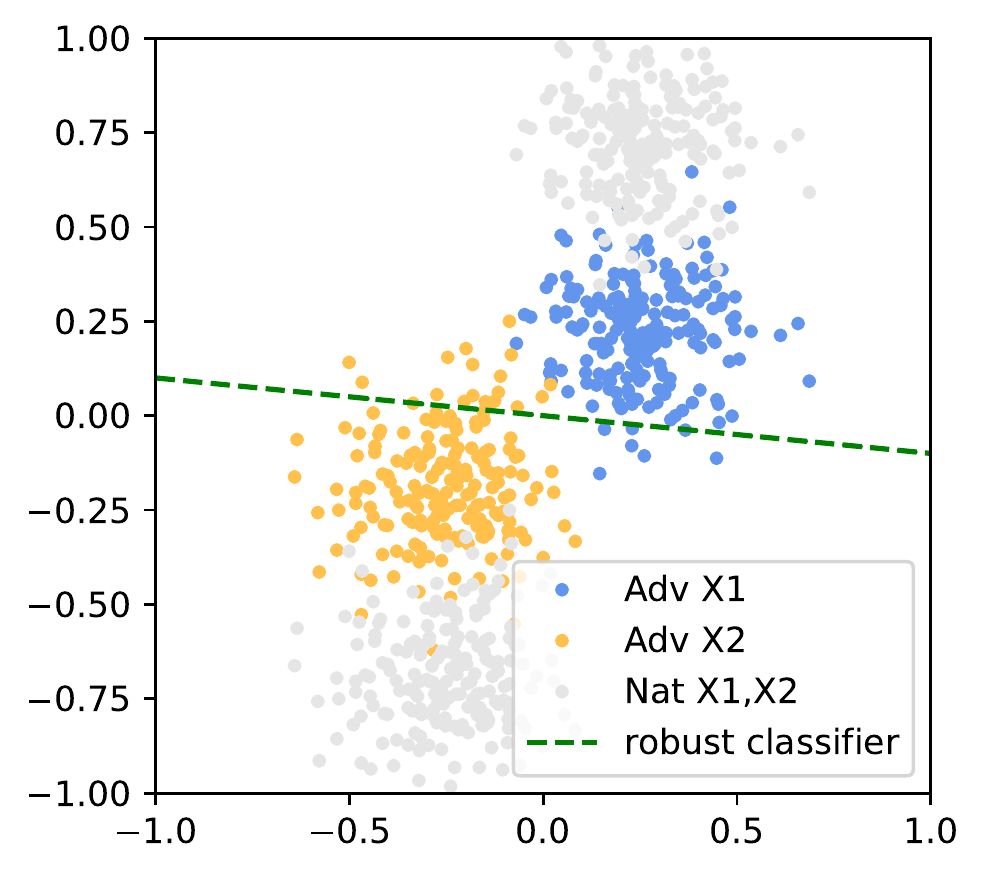}	
			\vspace{-6pt}
		\end{minipage}
		\begin{minipage}[t]{.24\linewidth}
			\centering
			\includegraphics[height=3.4cm]{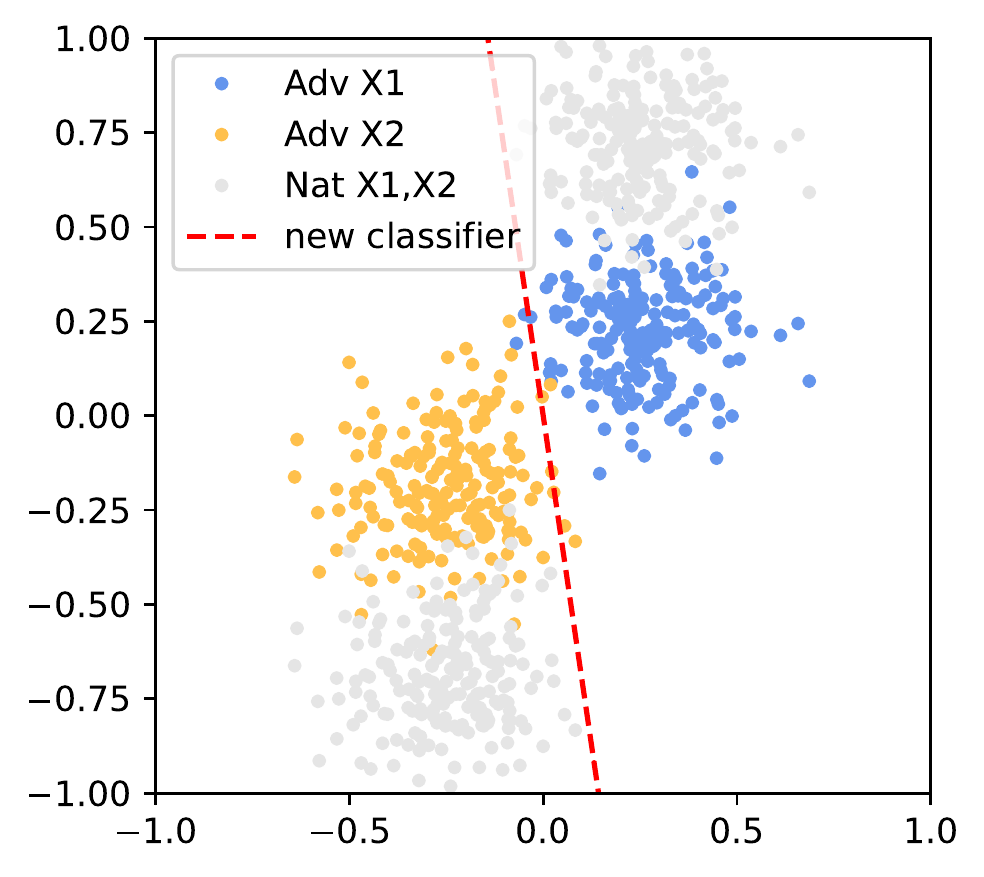}	
			\vspace{-6pt}
		\end{minipage}
		\begin{minipage}[t]{.24\linewidth}
			\centering
			\includegraphics[height=3.4cm]{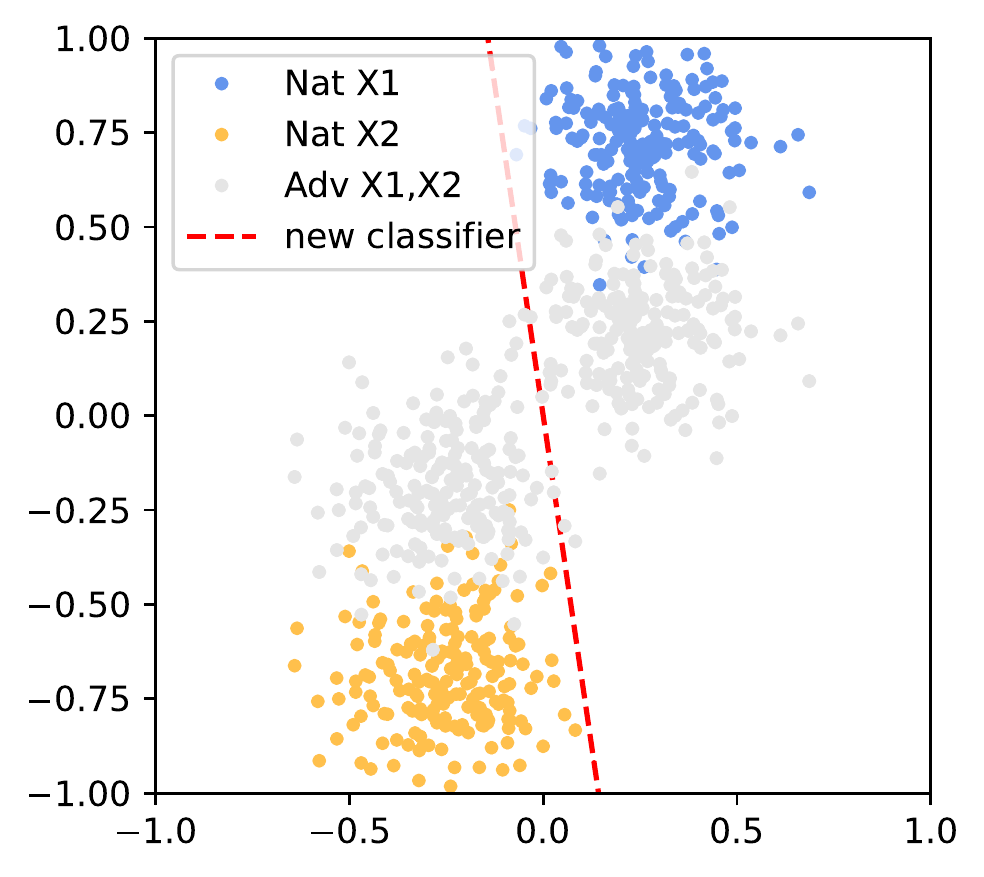}	
			\vspace{-6pt}
		\end{minipage}
		
		\caption{The process (from left to right) on how natural training on adversarial examples of a robust classifier leads to a non-robust model. The green line is the robust classifier used to generate adversarial examples, and the red line is the new classifier naturally trained on the adversarial examples of a robust classifier. ``Nat $X1$'' and ``Nat $X2$'' stand for natural data of class 1 and 2, respectively, and ``Adv $X1$'' and ``Adv $X2$'' are their adversarial counterparts. From the 4th plot we can see that the new classifier is non-robust.}
		\label{fig:advrobust}
		\vspace{-3pt}
	\end{figure*}

% \textbf{Synthetic data}
\begin{table*}[]
\caption{Experiments with different size of robust dataset on MNIST, CIFAR10, and TinyImageNet, where we naturally train classifiers on the datasets created by different methods.}
\vspace{-5pt}
\label{tab:diffsize}
\resizebox{1.0\textwidth}{!}{%
\begin{tabular}{clcccc}
\bottomrule
\multirow{2}{*}{Robust acc (\%) / Natural acc (\%)}          &  \multirow{2}{*}{Threat model}        & \multirow{2}{*}{Natural dataset} & \multicolumn{3}{c}{Robust dataset (ours)}\\
& & & 10\% size   & 20\% size   & 100\% size  \\ \hline\hline
\multirow{2}{*}{MNIST}        & 0.2 ($\ell_\infty$)   & 8.28/98.10      & 25.65/96.11 & 41.87/96.83 & 52.36/97.29 \\
                              & 2.0 ($\ell_2$)          & 21.28/98.10     & 34.47/96.75 & 40.21/97.59 & 48.40/98.25 \\ \hline
\multirow{2}{*}{CIFAR10}        & 2/255 ($\ell_\infty$) & 6.28/93.23      & 37.69/83.26 & 43.87/86.10 & 54.74/87.19 \\
                              & 0.25 ($\ell_2$)       & 9.85/93.23      & 42.61/81.50  & 45.54/82.80  & 59.52/86.59 \\ \hline
\multirow{2}{*}{TinyImageNet} & 2/255 ($\ell_\infty$) & 0.34/70.94      & 15.64/60.58 & 17.75/60.35 & 25.43/60.02 \\
                              & 0.25 ($\ell_2$)       & 4.55/70.94      & 24.38/63.42 & 26.71/62.71 & 39.55/61.13 \\ \bottomrule
\end{tabular}
% \begin{tabular}{c|cccccccc}
% \bottomrule
% Robust / Natural acc & \multicolumn{2}{c}{natural dataset}   & \multicolumn{2}{c}{10\% size of robust dataset} & \multicolumn{2}{c}{20\% size of robust dataset} & \multicolumn{2}{c}{100\% size of robust dataset} \\ \hline
%                      & 0.2 ($\ell_\infty$)   & 2.0 ($\ell_2$)    & 0.2 ($\ell_\infty$)         & 2.0 ($\ell_2$)          & 0.2 ($\ell_\infty$)         & 2.0 ($\ell_2$)          & 0.2 ($\ell_\infty$)          & 2.0 ($\ell_2$)          \\ \hline
% MNIST                & 8.28/98.10          & 21.28/98.10   & 25.65/96.11               & 34.47/96.75         & 41.87/96.83               & 40.21/97.59         & 52.36/97.29                & 48.40/98.25         \\ \hline\hline
%                      & 2/255 ($\ell_\infty$) & 0.25 ($\ell_2$) & 2/255 ($\ell_\infty$)       & 0.25 ($\ell_2$)       & 2/255 ($\ell_\infty$)       & 0.25 ($\ell_2$)       & 2/255 ($\ell_\infty$)        & 0.25 ($\ell_2$)       \\ \hline
% CIFAR                & 6.28/93.23          & 9.85/93.23    & 37.69/83.26               & 42.61/81.5          & 43.87/86.10               & 45.54/82.8          & 54.74/87.19                & 59.52/86.59         \\
% TinyImageNet         & 0.34/70.94                     &    4.55/70.94           &         15.64/60.58              &      24.38/63.42               &       17.75/60.35                   &          26.71/62.71           &            25.43/60.02                &      39.55/61.13               \\ \hline
% \end{tabular}
}
\end{table*}
% \subsubsection{MNIST}\label{sec:MNIST}
\medskip
\noindent\textbf{Experiment setup.} We use a CNN which has two convolutional layers, followed by two fully-connected layers for MNIST. We apply ResNet-18 for CIFAR10 and TinyImageNet. The output size of the last layer is the number of classes of each dataset. 
In our robust dataset learning algorithm (\autoref{alg:1}), we set $\theta_0$ to be the weights of the classifier. 
During training, we use PGD-20 to generate adversarial samples for MNIST and PGD-10 for the other two datasets. 
The step size of PGD attack is selected as $\epsilon/10$. 
Besides, we set the learning rate of the classifier to $\gamma=0.01$ and the learning rate of the robust dataset to $\beta = 0.5/255$. 
For the baseline methods, the datasets are generated using PGD attacks on natural and robust pre-train models (See Adv. data of natural and robust classifier in \autoref{tab:robust dataset}). 
To generate robust pre-train models, we train TRADES \cite{zhang2019theoretically} using the corresponding budgets. 
During evaluation phase, we use SGD with learning rate 0.01, momentum 0.9, and weight decay 5e-4, to fine tune the same model on all datasets.

\medskip
\noindent\textbf{Result analysis.} \autoref{tab:robust dataset} illustrates the experimental results on the three datasets. Compared to the baseline methods, the classifier naturally trained on our robust dataset achieves nearly 10\% increase on the robust accuracy on all tasks and attacks. We also notice that the classifier trained on the adversarial examples of robust classifier suffers from poor robust accuracy. We provide a simple example (\autoref{fig:advrobust}) to show that natural training on adversarial examples of a robust classifier may lead to a non-robust model.

\medskip
\noindent\textbf{Why do we not compare with adversarial training?} There are two reasons: 1) while the output of adversarial training is a classifier, in the robust dataset learning task, the input of our evaluation benchmark is a dataset on which we would natually train a classifier. Therefore, adversarial training does not fit our evaluation benchmark. Instead, we modify adversarial training as another baseline ``adversarial data of robust classifier'' in Table \ref{tab:robust dataset}. 2) We remark that our work is not aiming to set a new SOTA benchmark for adversarial defense, but rather to design a time-efficient method that benefits scenarios with limited computational resources. For example, learning a robust CIFAR10 dataset takes around 2 hours on a NVIDIA RTX A5000 GPU; fine tuning a classifier on our robust dataset takes at most 10 minutes. However, adversarial training, e.g., PGD Adversarial Training~\cite{madry2017towards}, TRADES \cite{zhang2019theoretically}, takes more than one day on the same GPU.

% \textbf{Result analysis.}
% \begin{table}[t]
% \label{TINYIMGNET}
% \caption{Tiny Imagenet results (Under PGD)}
% \centering
% %\resizebox{1.0\textwidth}{!}{%
% \begin{tabular}{l|cccc}
% \hline
% Robust acc / Natural acc         & 2/255 ($\ell_\infty$)  & 4/255 ($\ell_\infty$)  & 0.25 ($\ell_2$)  & 0.5 ($\ell_2$)  \\ \hline\hline
% Natural dataset                       & 0.34/70.94          & 0.16/70.94          & 4.55/70.94    & 0.52/70.94   \\ 
% Adv. data of natural classifier & 12.96/65.22         & 5.10/65.13         & 29.97/65.98   & 9.36/64.14  \\  \hline
% % Adv. data of robust classifier  &                     &                     &               &              \\ \hline
% Robust dataset (ours)                       & \textbf{25.43}/60.02         & \textbf{18.42}/60.36         & \textbf{39.55}/61.13   & \textbf{25.48}/60.92  \\ \hline
% \end{tabular}
% %}
% \end{table}

\begin{figure*}[t]
		\begin{minipage}[t]{.24\linewidth}
			\centering
			\includegraphics[height=2.8cm]{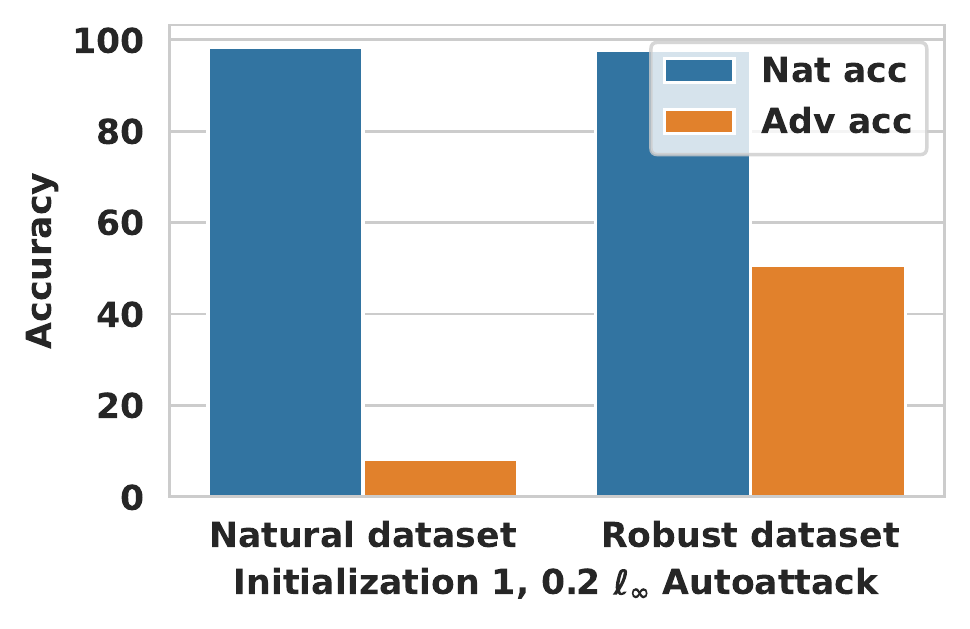}	
			\vspace{-6pt}
		\end{minipage}
		\begin{minipage}[t]{.24\linewidth}
			\centering
			\includegraphics[height=2.8cm]{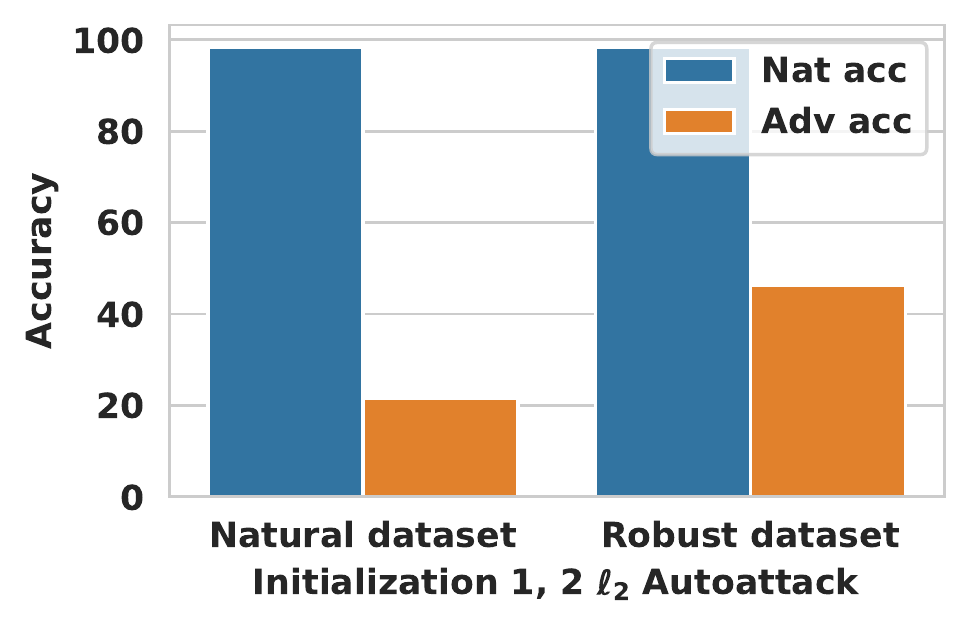}	
			\vspace{-6pt}
		\end{minipage}
		\begin{minipage}[t]{.24\linewidth}
			\centering
			\includegraphics[height=2.8cm]{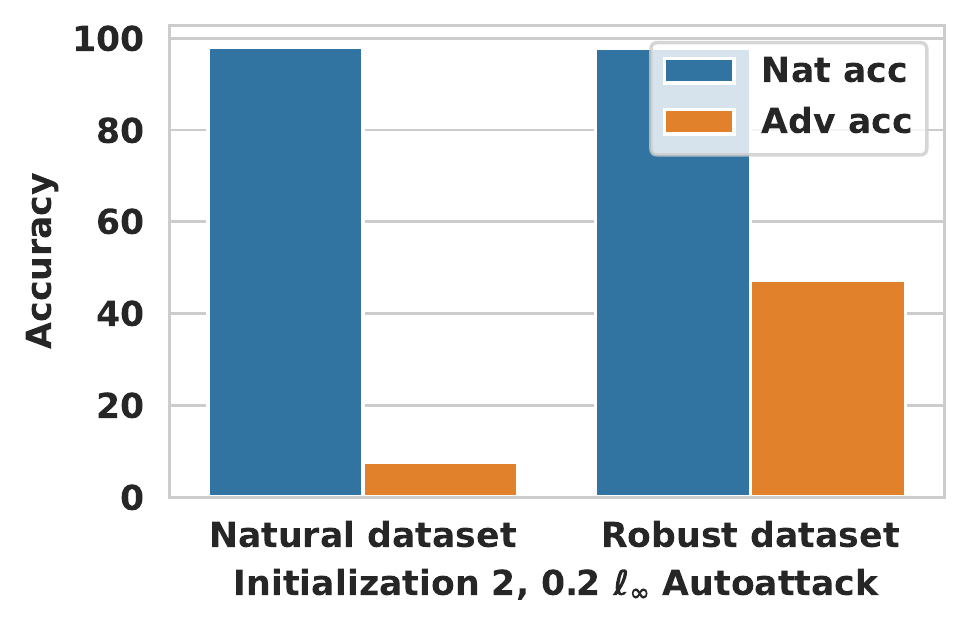}	
			\vspace{-6pt}
		\end{minipage}
		\begin{minipage}[t]{.24\linewidth}
			\centering
			\includegraphics[height=2.8cm]{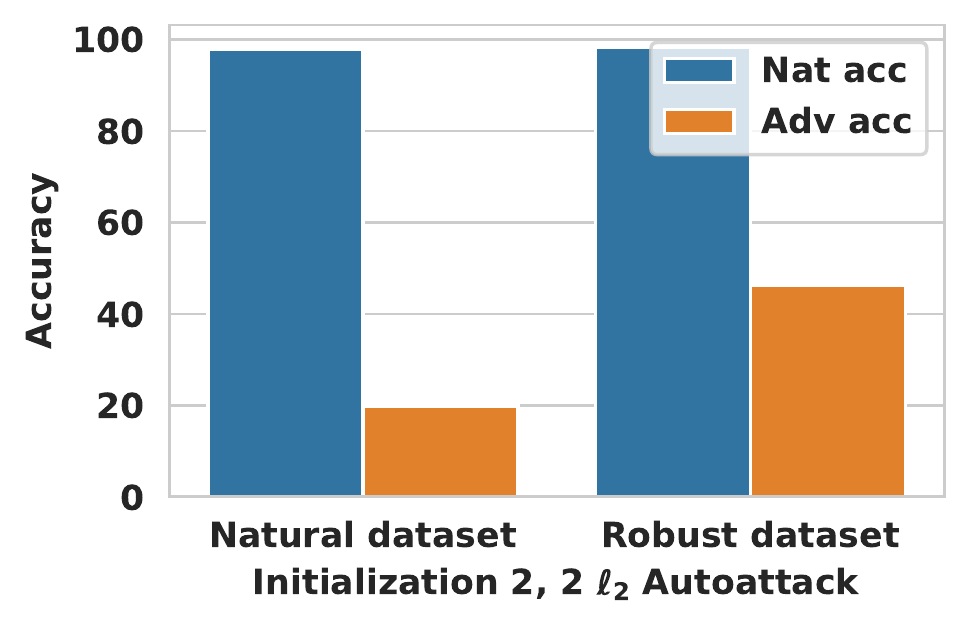}	
			\vspace{-6pt}
		\end{minipage}
	\\
% 	\subfigure*[]{
		\begin{minipage}[t]{.24\linewidth}
			\centering
			\includegraphics[height=2.9cm]{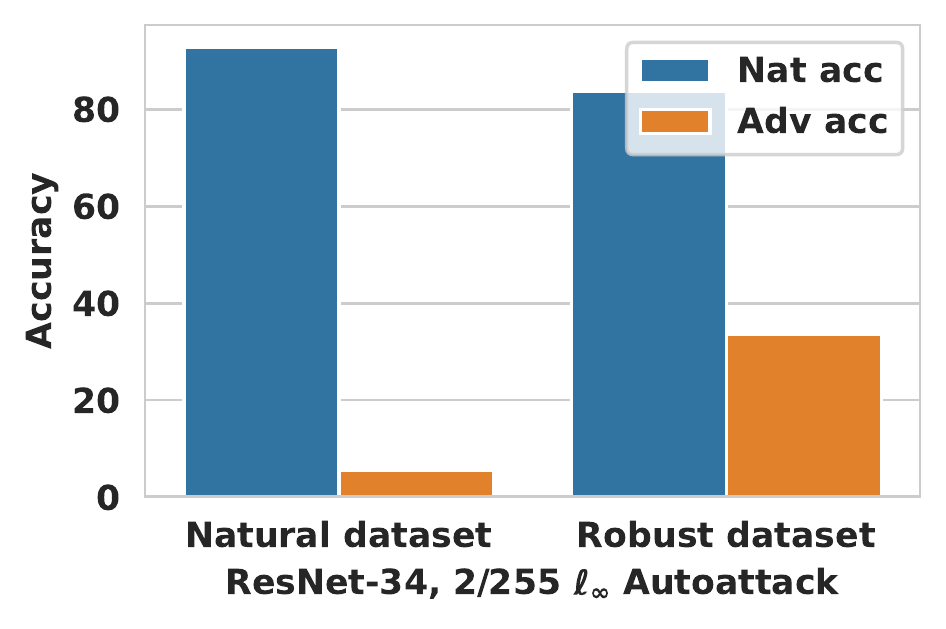}	
			\vspace{-6pt}
		\end{minipage}
% 	}
% 	\subfigure[]{
		\begin{minipage}[t]{.24\linewidth}
			\centering
			\includegraphics[height=2.9cm]{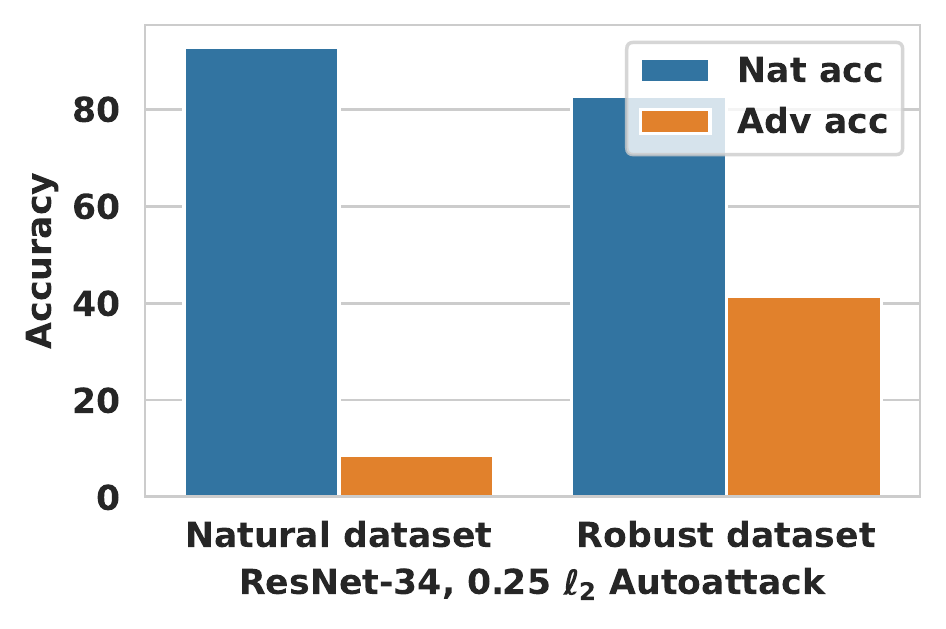}	
			\vspace{-6pt}
		\end{minipage}
% 	}
% 	\subfigure[]{
		\begin{minipage}[t]{.24\linewidth}
			\centering
			\includegraphics[height=2.9cm]{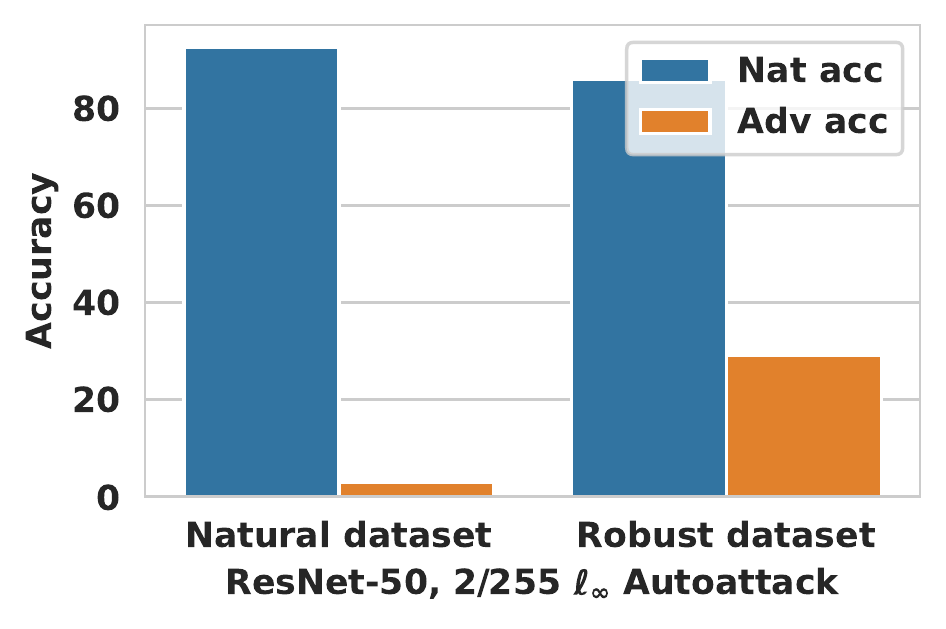}	
			\vspace{-6pt}
		\end{minipage}
		\begin{minipage}[t]{.24\linewidth}
			\centering
			\includegraphics[height=2.9cm]{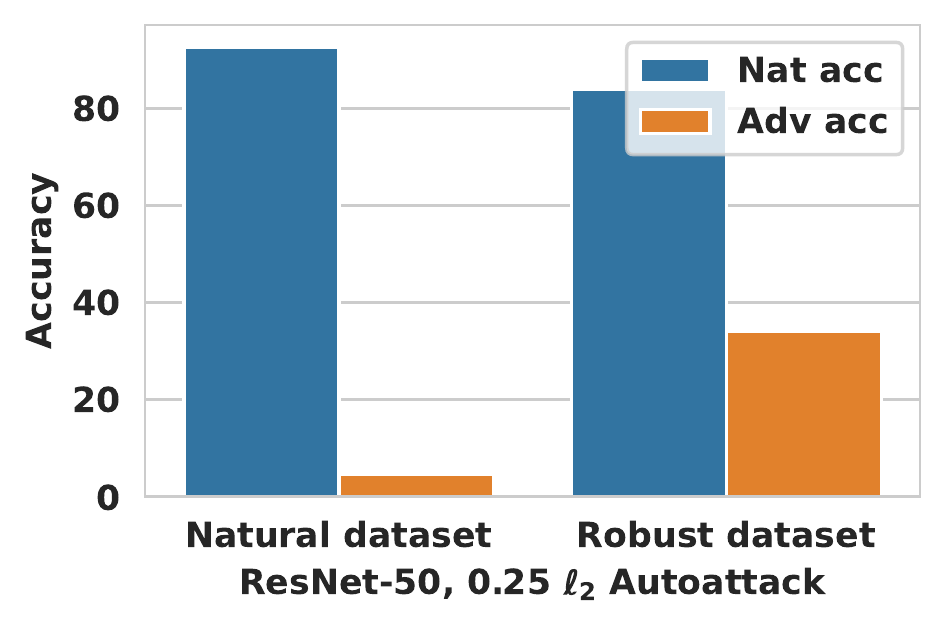}	
			\vspace{-6pt}
		\end{minipage}
% 	}
	\vspace{-10pt}

\caption{Experiments of transferability of our robust dataset. \textbf{Top:} Different initialization of the same CNN model on MNIST. We fine tune two CNN models (initialization 1 and initialization 2) on the robust MNIST dataset generated by a third initialization. \textbf{Bottom:} Different architectures on CIFAR10. We fine tune one ResNet-34 and one ResNet-50 models on the robust CIFAR10 generated by ResNet-18. }
\vspace{-12pt}
\label{fig:diffarch}
\end{figure*}

\subsection{Ablation study}
In this part, we conduct ablation experiments to study the effect of dataset size and the transferability of our robust dataset to different initialization and architectures. The robust dataset learning settings and evaluation methods are the same as in \autoref{sec:robustness}.

\medskip
\noindent\textbf{Dataset size.} We study the effect of different robust dataset size, where the robust dataset size is constrained to only 10\% and 20\% of the original dataset. We evaluate the robustness of classifiers under several adversarial budges for MNIST, CIFAR10 and TinyImageNet. The results are shown in \autoref{tab:diffsize}. We find that even if trained with 10\% of the size of the original dataset, the resulting classifier can still achieve 42.61\% robust accuracy on CIFAR10 and 24.38\% robust accuracy on TinyImageNet.

\medskip
\noindent\textbf{Different network initializations.} We evaluate the transferability of our robust dataset by using different network initializations. In our experiments, we apply the same CNN with different initializations to evaluate the classifier trained with MNIST robust dataset under 0.2 $\ell_{\infty}$ and  2.0 $\ell_{2}$ Autoattacks. In \autoref{fig:diffarch} Top, we see that our robust dataset effectively improves the robustness of the classifier with different network initializations.

\medskip
\noindent\textbf{Different network architectures.} We also investigate the case where the naturally trained network during testing has a different architecture from that we use to learn the robust dataset. In the experiments, we use our robust dataset to fine-tune ResNet-34 and ResNet-50 models. We evaluate the adversarial robustness under 2/255 $\ell_{\infty}$ and  0.25 $\ell_{2}$ Autoattack. In \autoref{fig:diffarch} Bottom, we see that our robust dataset enjoys descent transferability across different network architectures.

\medskip
\noindent\textbf{Different random seeds.} Due to the expensive cost of the Autoattack, we are not able to report the confidence score in \autoref{tab:robust dataset}. We repeat the experiments in \autoref{tab:robust dataset} with 10 different random seeds on 0.2 $\ell_\infty$ Autoattack on MNIST and 2/255 $\ell_\infty$ Autoattack on CIFAR10, the results are reported in \autoref{tab:seed}. From the table we see that our model share similar performance on different random seeds.
% \begin{table}[]
% \caption{MNIST initialization, should be in histogram.}
% \begin{tabular}{|l|ll|ll|}
% \hline
%                          & \multicolumn{2}{l|}{Initialization 1}                & \multicolumn{2}{l|}{Initialization 2}                \\ \hline
% Robust acc / Natural acc & \multicolumn{1}{l|}{0.2 $\ell_\infty$} & 2 $\ell_2$  & \multicolumn{1}{l|}{0.2 $\ell_\infty$} & 2 $\ell_2$  \\ \hline
% natural classifier       & \multicolumn{1}{l|}{8.26/98.33}        & 21.57/98.33 & \multicolumn{1}{l|}{7.65/97.92}        & 19.86/97.92 \\ \hline
% Our robust dataset       & \multicolumn{1}{l|}{50.68/97.76}       & 46.35/98.31 & \multicolumn{1}{l|}{47.32/97.79}       & 46.30/98.24 \\ \hline
% \end{tabular}
% \end{table}

% \begin{table}[]
% \caption{CIFAR archi, should be in histogram.}
% \begin{tabular}{|l|ll|ll|}
% \hline
%                          & \multicolumn{2}{l|}{ResNet-34}                           & \multicolumn{2}{l|}{ResNet-50}                           \\ \hline
% Robust acc / Natural acc & \multicolumn{1}{l|}{2/255 $\ell_\infty$} & 0.25 $\ell_2$ & \multicolumn{1}{l|}{2/255 $\ell_\infty$} & 0.25 $\ell_2$ \\ \hline
% natural classifier       & \multicolumn{1}{l|}{5.59/92.83}          & 8.51/92.83    & \multicolumn{1}{l|}{2.99/92.50}          & 4.66/92.60    \\ \hline
% Our robust dataset       & \multicolumn{1}{l|}{33.60/83.78}         & 41.48/82.69   & \multicolumn{1}{l|}{29.21/85.92}         & 34.09/83.81   \\ \hline
% \end{tabular}
% \end{table}

% \subsection{Randomized smoothing}

% \subsection{Visualization}
\begin{table}[]
\caption{Experiments with different random seeds on CIFAR10 and MNIST.}
\vspace{-5pt}
\label{tab:seed}
\begin{tabular}{lc}
\bottomrule
Robust dataset (ours)                      & MNIST                       \\ \hline
\begin{tabular}[c]{@{}c@{}}Robust acc / Natural acc\\ 0.2 ($\ell_\infty$)\end{tabular}  & 52.51 ± 0.59 / 98.14 ± 0.41 \\ \hline
                                                                                        & CIFAR10                     \\ \hline
\begin{tabular}[c]{@{}c@{}}Robust acc / Natural acc\\ 2/255($\ell_\infty$)\end{tabular} & 54.37 ± 0.71/ 97.29 ± 0.55  \\ \bottomrule
\end{tabular}
\vspace{-10pt}
\end{table}
% \begin{table}[]
% \begin{tabular}{lc}
% \hline
%                          & MNIST                       \\ \hline
% Robust acc / Natural acc & 0.2 ($\ell_\infty$)         \\ 
% Robust dataset (ours)    & 52.51 ± 0.59 / 98.14 ± 0.41 \\ \hline
%                          & CIFAR10                     \\ \hline
% Robust acc / Natural acc & 2/255($\ell_\infty$)        \\ 
% Robust dataset (ours)    & 54.37 ± 0.71/ 97.29 ± 0.55  \\ \hline
% \end{tabular}
% \end{table}
\vspace{-1pt}
\section{Conclusion}
\vspace{-1pt}
In this work, we propose a principle, tri-level optimization algorithm to solve the robust dataset learning problem. We theoretically prove the guarantee of our algorithm on an abstraction model, and empirically verify its effectiveness and efficiency on three popular image classification datasets. 
Our proposed algorithm provides a principled way to integrate the property of adversarial robustness into a dataset.
The evaluation results of the method imply various real-world applications under scenarios with limited computational resources.
% We hope our robust dataset can be applied to obtain a robust model under scenarios with limited computational resources.
%%%%%%%%% REFERENCES
{\small
\bibliographystyle{ieee_fullname}
\bibliography{main}
}

\clearpage
\appendix
\onecolumn
\section{Additional proofs}\label{sec:missing proof}
\begin{lemma}(Lemma D.1 in \cite{tsipras2018robustness})\label{lm:odds1}
The optimal solution $\w^*=(w_1^*,...,w_{d+1}^*)$ of our optimization problem \autoref{eq:svm} must satisfy $w_2=...=w_{d+1}$ and $\sign(w_i)\geq0,i\in[d+1]$.
\end{lemma}
\begin{proof}
We prove this lemma by contradiction, assume w.l.o.g. the optimal solution $\w^* = (w_1^*,...,w_{d+1}^*)$ satisfying $w_2^*\neq w_3^*$, we can let $\w' = (w_1^*,w_3^*,w_2^*,w_4^*,...,w_{d+1}^*)$.  In this case, we have $y\w^{*T}\x = y\w'^{T}\x$ because both $w_2^*x_2+w_3^*x_3$ and $w_2^*x_3+w_3^*x_2$ follow $\cN((w_2^*+w_3^*)\mu y,w_2^{*2}+w_3^{*2})$. So we have $\cL(\w';\cD)=\cL(\w^*;\cD)$. Moreover, since the margin term $\bbE_{(\x,y)}[\max\{0,1-y\w^T\x\}]$ is convex in $\w$, by Jensen’s inequality, averaging $\w^*$ and $\w'$ will not increase the value of that margin term.
On the other hand, $||\frac{\w^*+ \w'}{2}||_2^2< ||\w^*||_2^2$ as $2(\frac{w_2^*+w_3^*}{2})^2< w_2^{*2}+w_3^{*2}$ when $w_2^*\neq w_3^*$. Thus $\cL(\frac{\w^*+ \w'}{2};\cD)<\cL(\w^*;\cD)$, which yields contradiction. Analogously, if there exists $i$, such that $\sign(w^*_i)<0$, let  $\w' = (w_1^*,...,-w_i^*,...,w_{d+1}^*)$, we have $||\w'||_2^2 = ||\w^*||_2^2$ and $\bbE_{(\x,y)}[\max\{0,1-y\w'^T\x\}]\leq \bbE_{(\x,y)}[\max\{0,1-y\w^{*T}\x\}]$, which yields another contradiction.
\end{proof}

\begin{lemma}(Lemma D.2 in \cite{tsipras2018robustness})\label{lm:odds2}
If $\mu\geq\frac{4}{\sqrt{d}}$ and $p\leq 0.975$, the optimal solution $\w^*=(w_1^*,...,w_{d+1}^*)$ of our optimization problem \autoref{eq:svm} must satisfy $w_1^*<\sqrt{d}w_2^*$.
\end{lemma}
\begin{proof}
Assume for the sake of contradiction that $w_1^*\geq\sqrt{d}w_2^*$. By \autoref{lm:odds1} we have $ w_2^*=...=w_{d+1}^*$. Assume w.l.o.g. $||\w^*||_2=1$, we have $w_2^*\leq \frac{1}{\sqrt{2d}}$. Then, with probability at least $1 - p$, the first feature predicts the wrong label and without enough weight, the remaining features cannot compensate for it. Concretely,
\begin{equation}
\begin{split}
   \bbE[\max(0,1-y\w^{*T}\x)]&\geq (1-p)\bbE[\max(0,1+w_1^*-w_2^*\sum_{i=2}^{d+1}\cN(\mu,1))]\\
   &\geq(1-p)\bbE[\max(0,1+\sqrt{d}w_2^*-w_2^*\cN(d\mu,d))]\\
   &\geq(1-p)\bbE[\max(0,1+1/\sqrt{2}-\cN(\sqrt{\frac{d}{2}}\mu,\frac{1}{2}))].\\
   \end{split}
\end{equation}
We will now show that a solution $\w' = (w_1',...,w_{d+1}')$ that assigns zero weight on the first feature ($w_2'=...=w_{d+1}'=\frac{1}{\sqrt{d}}$ and $w_1' = 0$), achieves a better loss. 
\begin{equation}
\begin{split}
   \bbE[\max(0,1-y\w'^{T}\x)]&=\bbE[\max(0,1-\cN(\sqrt{d}\mu,1))].\\
   \end{split}
\end{equation}
By the optimality of $\w^*$ we must have $$\bbE[\max(0,1-\cN(\sqrt{d}\mu,1))] \geq(1-p)\bbE[\max(0,1+1/\sqrt{2}-\cN(\sqrt{\frac{d}{2}}\mu,\frac{1}{2}))],$$ which yields
$p\geq 1-\frac{\bbE[\max(0,1-\cN(\sqrt{d}\mu,1))]}{\bbE[\max(0,1+1/\sqrt{2}-\cN(\sqrt{\frac{d}{2}}\mu,\frac{1}{2}))]}\geq 1-\frac{\bbE[\max(0,1-\cN(4,1))]}{\bbE[\max(0,1+1/\sqrt{2}-\cN(2\sqrt{2},\frac{1}{2}))]}>0.975,$ which contradicts to the condition that $p\leq 0.975$.
\end{proof}
\subsection{Proof of \autoref{thm:natural train}}
\begin{proof}
\textbf{Part 1.} We show that the optimal classifier can achieve high natural accuracy.
By \autoref{lm:odds1} and \autoref{lm:odds2} we have $w_2^*=...=w_{d+1}^*$, $\sign(w_i^*)\geq0,i\in[d+1]$, and $w_1^*\leq\sqrt{d}w_2^*$.
Consider $\w^{*T}\x = w_1^*x_1+w_2^*\sum_{i=2}^{d+1}\cN(\mu y,1) = w_2^*\cN(\frac{w_1^*}{w_2^*}x_1+d\mu y,d)$, because $\epsilon\geq 2\mu$ and $\mu\geq \frac{4}{\sqrt{d}}$, the probability that $\x$ is correctly classified is
\begin{equation}
    \begin{split}
        \Pr(\sign(\w^{*T}\x)=\sign(y)) &= p\Pr(\cN(\frac{w_1^*}{w_2^*}+d\mu,d)>0)\\
        &+(1-p)\Pr(\cN(-\frac{w_1^*}{w_2^*}+d\mu,d)>0)\\
        &\geq p\Pr(\cN(\mu d,d)>0)+(1-p)\Pr(\cN(\mu-\sqrt{d},d)>0)\\
        &\geq p\Pr(\cN(4\sqrt{d},d)>0)+(1-p)\Pr(\cN(3\sqrt{d},d)>0)\\
        &= p\Pr(\cN(4,1)>0)+(1-p)\Pr(\cN(3,1)>0)\\
        &\geq \Pr(\cN(3,1)>0)\\
        &\geq 0.9986.
    \end{split}
\end{equation}
Thus the natural accuracy of the optimal classifier $\w^*$ is greater than 99\%.

\textbf{Part 2.} We show that the optimal classifier achieve low robust accuracy.
Firstly, according to \autoref{lm:pertubation}, the perturbed distribution $\x+\bm\delta$ is given by  
\begin{equation}
    y\sim\{-1,1\},\quad x_1\sim\left\{\begin{aligned}
         y(1-\epsilon),\quad &\textrm{with probability } p;\\
        - y(1+\epsilon),\quad &\textrm{with probability } 1-p,
    \end{aligned}\right.\quad
    x_i\sim \cN((\mu-\epsilon) y,1),\ i\geq2.
\end{equation}
By \autoref{lm:odds1} and \autoref{lm:odds2}, we have $w_2^*=...=w_{d+1}^*$, $\sign(w_i^*)\geq0,i\in[d+1]$, and $w_1^*\leq\sqrt{d}w_2^*$.
Consider $\w^{*T}(\x+\bm\delta) = w_1^*x_1+w_2^*\sum_{i=2}^{d+1}\cN((\mu-\epsilon)y,1) = w_2^*\cN(\frac{w_1^*}{w_2^*}x_1+d(\mu-\epsilon)y,d)$. Because $\epsilon\geq 2\mu$ and $\mu\geq \frac{4}{\sqrt{d}}$, the probability that $\x+\delta$ is correctly classified is
\begin{equation}
    \begin{split}
        \Pr(\sign(\w^{*T}\x)=\sign(y)) &= p\Pr(\cN(\frac{w_1^*}{w_2^*}+d(\mu-\epsilon),d)>0)\\
        &+(1-p)\Pr(\cN(-\frac{w_1^*}{w_2^*}+d(\mu-\epsilon),d)>0)\\
        &\leq p\Pr(\cN(\sqrt{d}-\mu d,d)>0)+(1-p)\Pr(\cN(-\mu d,d)>0)\\
        &\leq p\Pr(\cN(-3\sqrt{d},d)>0)+(1-p)\Pr(\cN(-4\sqrt{d},d)>0)\\
        &= p\Pr(\cN(-3,1)>0)+(1-p)\Pr(\cN(-4,1)>0)\\
        &\leq \Pr(\cN(-3,1)>0)\\
        &\leq 0.00135.
    \end{split}
\end{equation}
Thus the robust accuracy of the optimal classifier $\w^*$ is less than 0.2\%.
\end{proof}

\subsection{Proof of \autoref{lm:optweight}}
\begin{proof}
The optimal $\cD^*$ should make $1-y\w^{T}_{\cD^*}\x+\epsilon||\w_{\cD^*}||_1$ and $||\w_{\cD^*}||_2^2$ as small as possible. \begin{equation}\label{eq:expansion}
1-y\w^{T}_{\cD^*}\x+\epsilon||\w_{\cD^*}||_1=1-yw_{\cD^*}^{(1)}x_1+\epsilon |w_{\cD^*}^{(1)}|+ \sum_{i=2}^{d+1}(\epsilon|w_{\cD^*}^{(i)}|-yw_{\cD^*}^{(i)}x_i),
\end{equation}

as $x_i\sim \cN(\mu y,1),i\geq 2$, we have $\sum_{i=2}^{d+1}(\epsilon|w_{\cD^*}^{(i)}|-yw_{\cD^*}^{(i)}x_i)\sim \cN(\sum_{i=2}^{d+1}(\epsilon|w_{\cD^*}^{(i)}|-\mu yw_{\cD^*}^{(i)}),\sum_{i=2}^{d+1}w_{\cD^*}^{(i)2})$. 

Denote by $\cL_{\cD'}:=\bbE_{(\x,y)}[\max\{0,1-y\w^{T}_{\cD'}\x+\epsilon||\w_{\cD'}||_1\}$. Assume there exist $w_{\cD^*}^{(i)}\neq0,i\geq2$, we will show there exists $\cD_0$ such that $\w_{\cD_0} := (w_{\cD^*}^{(1)},0,...,0)$ and
\begin{equation}
\begin{split}
    \cL_{\cD_0}+\lambda||\w_{\cD_0}||_2^2<\cL_{\cD^{*}}+\lambda||\w_{\cD^*}||_2^2.
\end{split}
\end{equation}

\textbf{Step 1:} We will show if $\w_{\cD_0} := (w_{\cD^*}^{(1)},0,...,0)$, $$\cL_{\cD_0}+\lambda||\w_{\cD_0}||_2^2<\cL_{\cD^{*}}+\lambda||\w_{\cD^*}||_2^2.$$
Firstly, it is easy to observe that $\lambda||\w_{\cD_0}||_2^2<\lambda||\w_{\cD^*}||_2^2$. Then we focus on the term $\cL_{\cD_0}$ and $\cL_{\cD'}$.

Denote by $A = 1-yw_{\cD^*}^{(1)}x_1+\epsilon |w_{\cD^*}^{(1)}|$, 
$\mu' = \sum_{i=2}^{d+1}(\epsilon|w_{\cD^*}^{(i)}|-\mu yw_{\cD^*}^{(i)})$, $\sigma'^2 = \sum_{i=2}^{d+1}w_{\cD^*}^{(i)2}$, and
$z= \sum_{i=2}^{d+1}(\epsilon|w_{\cD^*}^{(i)}|-yw_{\cD^*}^{(i)}x_i)\sim \cN(\mu',\sigma'^2)$. Then by \autoref{eq:expansion} we have
$$1-y\w^{T}_{\cD^*}\x+\epsilon||\w_{\cD^*}||_1 = A+z,$$ and thus we can simplify $\cL_{\cD_0}, \cL_{\cD^{*}}$ as below:
\begin{equation}
    \cL_{\cD_0} = \bbE_{(\x,y)}[\max\{0,1-y\w^{T}_{\cD_0}\x+\epsilon||\w_{\cD_0}||_1\}]=\bbE_{x_1,y}[\max\{0,1-yw_{\cD^*}^{(1)}x_1+\epsilon |w_{\cD^*}^{(1)}|\}] = \bbE_{x_1,y}[A\bbI_{A\geq 0}]
\end{equation}
\begin{equation}
    \cL_{\cD^{*}} = \bbE_{(\x,y)}[\max\{0,1-y\w^{T}_{\cD^*}\x+\epsilon||\w_{\cD^*}||_1\}]= \bbE_{\x,y}[(A+z)\bbI_{A+z\geq0}]
\end{equation}
Consider $\cL_{\cD^*}=\bbE_{\x,y}[(A+z)\bbI_{A+z\geq0}]$,
\begin{equation}\label{eq:expectation}
\begin{split}
    \cL_{\cD^*}=\bbE_{\x,y}[(A+z)\bbI_{A+z\geq0}] &\geq \bbE_{z,x_1,y}[(A+z)\bbI_{A+z\geq0}\bbI_{A\geq0}]\\
    &= \bbE_{x_1,y}[\bbI_{A\geq0}\bbE_{z}[(A+z)\bbI_{z\geq-A}]]\\
    &= \bbE_{x_1,y}[A\bbI_{A\geq0}\bbE_{z}[\bbI_{z\geq-A}]+\bbI_{A\geq0}\bbE_{z}[z\bbI_{z\geq-A}]]\\
    & = \bbE_{x_1,y}[A\bbI_{A\geq0}]-\bbE_{x_1,y}[A\bbI_{A\geq0}\bbE_{z}[\bbI_{z<-A}]]+\bbE_{x_1,y}[\bbI_{A\geq0}\bbE_{z}[z\bbI_{z\geq-A}]]
\end{split}
\end{equation}

Now we consider $\bbE_{z}[z\bbI_{z\geq-A}]$, as $z\sim \cN(\mu',\sigma'^2)$, we have $\frac{z-\mu'}{\sigma'}\sim\cN(0,1)$ and
\begin{equation}
\begin{split}
    \bbE_{z}[z\bbI_{z\geq-A}] &= \bbE_{z}[z\bbI_{\mu'\geq z\geq-A}]+\bbE_{z}[z\bbI_{2\mu'+A\geq z\geq\mu'}]+\bbE_{z}[z\bbI_{z\geq 2\mu'+A}]\\
    & = \bbE_{s\sim\cN(0,1)}[(\sigma's+\mu')\bbI_{0\geq s\geq-\frac{A+\mu'}{\sigma'}}]+\bbE_{s\sim\cN(0,1)}[(\sigma's+\mu')\bbI_{\frac{A+\mu'}{\sigma'}\geq s\geq0}]+
    \bbE_{z}[z\bbI_{z\geq 2\mu'+A}]\\
    & = 2\mu'\bbE_{s\sim\cN(0,1)}[\bbI_{0\geq s\geq-\frac{A+\mu'}{\sigma'}}]+
    \bbE_{z}[z\bbI_{z\geq 2\mu'+A}]
\end{split}
\end{equation}
since $\epsilon>\mu$, we have $$\mu' = \sum_{i=2}^{d+1}(\epsilon|w_{\cD^*}^{(i)}|-\mu yw_{\cD^*}^{(i)})\geq\sum_{i=2}^{d+1}(\epsilon-\mu)|w_{\cD^*}^{(i)}|>0$$
Thus
\begin{equation}\label{eq:z}
\begin{split}
\bbE_{z}[z\bbI_{z\geq-A}] &=  2\mu'\bbE_{s\sim\cN(0,1)}[\bbI_{0\geq s\geq-\frac{A+\mu'}{\sigma'}}]+
    \bbE_{z}[z\bbI_{z\geq 2\mu'+A}]\\
    &>\bbE_{z}[z\bbI_{z\geq 2\mu'+A}]\\
    &> (2\mu'+A)\bbE_{z}[\bbI_{z\geq 2\mu'+A}]\\
    &> A\bbE_{z}[\bbI_{z\geq 2\mu'+A}]\\
    &=A\bbE_{z}[\bbI_{z\leq -A}]
    \end{split}
\end{equation}
Plug \autoref{eq:z} into \autoref{eq:expectation} we have

\begin{equation}\label{eq:}
\begin{split}
    \cL_{\cD^*}=\bbE_{\x,y}[(A+z)\bbI_{A+z\geq0}] &\geq\bbE_{x_1,y}[A\bbI_{A\geq0}]-\bbE_{x_1,y}[A\bbI_{A\geq0}\bbE_{z}[\bbI_{z<-A}]]+\bbE_{x_1,y}[\bbI_{A\geq0}\bbE_{z}[z\bbI_{z\geq-A}]]\\
    &>\bbE_{x_1,y}[A\bbI_{A\geq0}]-\bbE_{x_1,y}[A\bbI_{A\geq0}\bbE_{z}[\bbI_{z<-A}]]+\bbE_{x_1,y}[A\bbI_{A\geq0}\bbE_{z}[\bbI_{z\leq -A}]]\\
    &= \bbE_{x_1,y}[A\bbI_{A\geq0}]\\
    &= \cL_{\cD_0}.
\end{split}
\end{equation}

\textbf{Step 2:} We will show the existence of distribution $\cD_0$ such that $\w_{\cD_0} := (w_{\cD^*}^{(1)},0,...,0)$,  we can set $x_2=x_3=...=x_{d+1}=1$ in $\cD_0$ and $x_1=\frac{1}{w_{\cD^*}^{(1)}}y$ such that $w_{\cD_0}^{(1)} = w_{\cD^*}^{(1)}$ and $w_{\cD_0}^{(i)}=0,i\geq2$.

Combining Step 1 and 2 yields contradiction.
\end{proof}

\subsection{Proof of \autoref{thm:robustdata}}
\begin{proof}
According to \autoref{lm:optweight}, the weight of the optimal SVM learned from $\cD^*$ should satisfy  $w_{\cD^*}^{(2)}=w_{\cD^*}^{(3)}=...=w_{\cD^*}^{(d+1)}=0$. Thus the clean and robust accuracy are only related to the first feature of $\x$. It is easy to see the natural accuracy is equal to the probability that $\sign(x_1)=y$, which is $p$. Besides, when the perturbation budget $\epsilon<1$, the adversary does not change the sign of $x_1$. Thus the robust accuracy is also $p$.
\end{proof}

\subsection{Advanced Theoretical Analysis on a General Dataset}\label{sec:general dataset}
% \{(\x^{(i)},y^{(i)})\in\bbR^{n\times1}\}
Consider dataset distribution $(\x,y)\in\bbR^{(d+1)\times1}$ follow the  distribution below:
\begin{equation}
    y\sim\{-1,1\},\quad x_1\sim\left\{\begin{aligned}
         y,\quad &\textrm{with prob } p;\\
        - y,\quad &\textrm{with prob } 1-p,
    \end{aligned}\right.\quad
    x_i\sim \cD_i,\ i\geq 2,
\end{equation}
where $\cD_i$ are symmetric distributions with mean $\mu_i\leq 1$.

\begin{lemma}\label{lm:symmetric}
The sum of independent symmetric distributions is also symmetric.
\end{lemma}
\begin{proof}
Based on the fact (which is easy to prove) that a random variable is symmetric if and only if its characteristic function is real-valued. The characteristic function of the sum of independent symmetric distributions is given by the multiplication of the characteristic function of independent symmetric distributions, which is also real-valued. Thus the sum of independent symmetric distributions is also symmetric.
\end{proof}

Following the settings in the above section we have the lemma below
\begin{lemma}\label{lm:optweight2}Our minmax optimization problem is
\begin{equation}
    \min_{\cD'}\max_{||\bm\delta||_\infty\leq \epsilon} \bbE_{(\x,y)}[\max\{0,1-y\w^{T}_{\cD'}(\x+\bm\delta)\}] = \min_{\cD'}\bbE_{(\x,y)}[\max\{0,1-y\w^{T}_{\cD'}\x+\epsilon||\w_{\cD'}||_1\}]
\end{equation}

% the optimal $\alpha,\mu,\sigma$ should satisfy $\mu=0,\sigma=0$
Let
\begin{equation}
    \cD^*= \argmin_{\cD'}\bbE_{(\x,y)}[\max\{0,1-y\w^{T}_{\cD'}\x+\epsilon||\w_{\cD'}||_1\}],
\end{equation}
if $1>\epsilon\geq\max_{i\geq2}\mu_i$ , then $\w_{\cD^*} := (w_{\cD^*}^{(1)},...,w_{\cD^*}^{(d+1)})$ must satisfy $w_{\cD^*}^{(2)}=w_{\cD^*}^{(3)}=...=w_{\cD^*}^{(d+1)}=0$
\end{lemma}
\begin{proof}
The optimal $\cD^*$ should make $1-y\w^{T}_{\cD^*}\x+\epsilon||\w_{\cD^*}||_1$ as small as possible. \begin{equation}\label{eq:expansion}
1-y\w^{T}_{\cD^*}\x+\epsilon||\w_{\cD^*}||_1=1-yw_{\cD^*}^{(1)}x_1+\epsilon |w_{\cD^*}^{(1)}|+ \sum_{i=2}^{d+1}(\epsilon|w_{\cD^*}^{(i)}|-yw_{\cD^*}^{(i)}x_i),
\end{equation}

as $x_i\sim \cN(\mu y,1),i\geq 2$, we have $\sum_{i=2}^{d+1}(\epsilon|w_{\cD^*}^{(i)}|-yw_{\cD^*}^{(i)}x_i)\sim \cN(\sum_{i=2}^{d+1}(\epsilon|w_{\cD^*}^{(i)}|-\mu yw_{\cD^*}^{(i)}),\sum_{i=2}^{d+1}w_{\cD^*}^{(i)2})$. 

Assume there exist $w_{\cD^*}^{(i)}\neq0,i\geq2$, we will show there exists $\cD_0$ such that $\w_{\cD_0} := (w_{\cD^*}^{(1)},0,...,0)$ and $$\cL_{\cD_0}:=\bbE_{(\x,y)}[\max\{0,1-y\w^{T}_{\cD_0}\x+\epsilon||\w_{\cD_0}||_1\}<\bbE_{(\x,y)}[\max\{0,1-y\w^{T}_{\cD^*}\x+\epsilon||\w_{\cD^*}||_1\}]=:\cL_{\cD^{*}}$$

\textbf{Step 1:} We will show if $\w_{\cD_0} := (w_{\cD^*}^{(1)},0,...,0)$, $$\cL_{\cD_0}< \cL_{\cD^{*}}.$$
Denote by $A = 1-yw_{\cD^*}^{(1)}x_1+\epsilon |w_{\cD^*}^{(1)}|$, 
$\mu' = \sum_{i=2}^{d+1}(\epsilon|w_{\cD^*}^{(i)}|-\mu yw_{\cD^*}^{(i)})$, and
$z= \sum_{i=2}^{d+1}(\epsilon|w_{\cD^*}^{(i)}|-yw_{\cD^*}^{(i)}x_i)\sim \cS$, by \autoref{lm:symmetric} we know $\cS$ is symmetric with mean $\mu'$. Then by \autoref{eq:expansion} we have
$$1-y\w^{T}_{\cD^*}\x+\epsilon||\w_{\cD^*}||_1 = A+z,$$ and thus we can simplify $\cL_{\cD_0}, \cL_{\cD^{*}}$ as below:
\begin{equation}
    \cL_{\cD_0} = \bbE_{(\x,y)}[\max\{0,1-y\w^{T}_{\cD_0}\x+\epsilon||\w_{\cD_0}||_1\}]=\bbE_{x_1,y}[\max\{0,1-yw_{\cD^*}^{(1)}x_1+\epsilon |w_{\cD^*}^{(1)}|\}] = \bbE_{x_1,y}[A\bbI_{A\geq 0}],
\end{equation}
\begin{equation}
    \cL_{\cD^{*}} = \bbE_{(\x,y)}[\max\{0,1-y\w^{T}_{\cD^*}\x+\epsilon||\w_{\cD^*}||_1\}]= \bbE_{\x,y}[(A+z)\bbI_{A+z\geq0}].
\end{equation}
Consider $\cL_{\cD^*}=\bbE_{\x,y}[(A+z)\bbI_{A+z\geq0}]$,
\begin{equation}\label{eq:expectation2}
\begin{split}
    \cL_{\cD^*}=\bbE_{\x,y}[(A+z)\bbI_{A+z\geq0}] &\geq \bbE_{z,x_1,y}[(A+z)\bbI_{A+z\geq0}\bbI_{A\geq0}]\\
    &= \bbE_{x_1,y}[\bbI_{A\geq0}\bbE_{z}[(A+z)\bbI_{z\geq-A}]]\\
    &= \bbE_{x_1,y}[A\bbI_{A\geq0}\bbE_{z}[\bbI_{z\geq-A}]+\bbI_{A\geq0}\bbE_{z}[z\bbI_{z\geq-A}]]\\
    & = \bbE_{x_1,y}[A\bbI_{A\geq0}]-\bbE_{x_1,y}[A\bbI_{A\geq0}\bbE_{z}[\bbI_{z<-A}]]+\bbE_{x_1,y}[\bbI_{A\geq0}\bbE_{z}[z\bbI_{z\geq-A}]]
\end{split}
\end{equation}

Now we consider $\bbE_{z}[z\bbI_{z\geq-A}]$, as $z\sim \cS$ and $\cS$ is symmetric with $\mu'$. We have $\cS-\mu'$ is symmetric with 0 and
\begin{equation}
\begin{split}
    \bbE_{z}[z\bbI_{z\geq-A}] &= \bbE_{z}[z\bbI_{\mu'\geq z\geq-A}]+\bbE_{z}[z\bbI_{2\mu'+A\geq z\geq\mu'}]+\bbE_{z}[z\bbI_{z\geq 2\mu'+A}]\\
    & = \bbE_{s\sim\cS-\mu'}[(s+\mu')\bbI_{0\geq s\geq-A-\mu'}]+\bbE_{s\sim\cS-\mu'}[(s+\mu')\bbI_{A+\mu'\geq s\geq0}]+
    \bbE_{z}[z\bbI_{z\geq 2\mu'+A}]\\
    & = 2\mu'\bbE_{s\sim\cS-\mu'}[\bbI_{0\geq s\geq-A-\mu'}]+
    \bbE_{z}[z\bbI_{z\geq 2\mu'+A}]\\
    & = 2\mu'\bbE_{z}[\bbI_{\mu'\geq z\geq-A}]+
    \bbE_{z}[z\bbI_{z\geq 2\mu'+A}].
\end{split}
\end{equation}
Since $\epsilon>\mu$, we have $$\mu' = \sum_{i=2}^{d+1}(\epsilon|w_{\cD^*}^{(i)}|-\mu yw_{\cD^*}^{(i)})\geq\sum_{i=2}^{d+1}(\epsilon-\mu)|w_{\cD^*}^{(i)}|>0.$$
Thus
\begin{equation}\label{eq:z2}
\begin{split}
\bbE_{z}[z\bbI_{z\geq-A}] &=  2\mu'\bbE_{z}[\bbI_{\mu'\geq z\geq-A}]+
    \bbE_{z}[z\bbI_{z\geq 2\mu'+A}]\\
    &>\bbE_{z}[z\bbI_{z\geq 2\mu'+A}]\\
    &> (2\mu'+A)\bbE_{z}[\bbI_{z\geq 2\mu'+A}]\\
    &> A\bbE_{z}[\bbI_{z\geq 2\mu'+A}]\\
    &=A\bbE_{z}[\bbI_{z\leq -A}].
    \end{split}
\end{equation}
Plugging \autoref{eq:z2} into \autoref{eq:expectation2}, we have

\begin{equation}\label{eq:}
\begin{split}
    \cL_{\cD^*}=\bbE_{\x,y}[(A+z)\bbI_{A+z\geq0}] &\geq\bbE_{x_1,y}[A\bbI_{A\geq0}]-\bbE_{x_1,y}[A\bbI_{A\geq0}\bbE_{z}[\bbI_{z<-A}]]+\bbE_{x_1,y}[\bbI_{A\geq0}\bbE_{z}[z\bbI_{z\geq-A}]]\\
    &>\bbE_{x_1,y}[A\bbI_{A\geq0}]-\bbE_{x_1,y}[A\bbI_{A\geq0}\bbE_{z}[\bbI_{z<-A}]]+\bbE_{x_1,y}[A\bbI_{A\geq0}\bbE_{z}[\bbI_{z\leq -A}]]\\
    &= \bbE_{x_1,y}[A\bbI_{A\geq0}]\\
    &= \cL_{\cD_0}.
\end{split}
\end{equation}

\textbf{Step 2:} We will show the existence of distribution $\cD_0$ such that $\w_{\cD_0} := (w_{\cD^*}^{(1)},0,...,0)$,  we can set $x_2=x_3=...=x_{d+1}=1$ in $\cD_0$ and $x_1=\frac{1}{w_{\cD^*}^{(1)}}y$ such that $w_{\cD_0}^{(1)} = w_{\cD^*}^{(1)}$ and $w_{\cD_0}^{(i)}=0,i\geq2$.

Combining Step 1 and 2 yields contradiction.
\end{proof}

\begin{lemma}
The optimal SVM learned from $\cD^*$ have at least $p$ clean and robust accuracy (with $\ell_\infty$ budget less than 1) on the original dataset.
\end{lemma}

\begin{proof}
According to \autoref{lm:optweight2}, the weight of the optimal SVM learned from $\cD^*$ should satisfy  $w_{\cD^*}^{(2)}=w_{\cD^*}^{(3)}=...=w_{\cD^*}^{(d+1)}=0$. Thus the clean and robust accuracy are only related to the first feature of $\x$. It is easy to see the natural accuracy is equal to the probability that $\sign(x_1)=y$, which is $p$. Besides, when the perturbation budget $\epsilon<1$, the adversary does not change the sign of $x_1$. Thus the robust accuracy is also $p$.
\end{proof}

\newpage
\section{Additional Experiment Results}
\label{sec:app_exp}

\begin{table}[ht]
\label{TINYIMGNET-AA}
\caption{TinyImagenet results (Under AutoAttack)}
\centering
%\resizebox{1.0\textwidth}{!}{%
\begin{tabular}{l|cc}
\hline
Robust acc / Natural acc         & 2/255 ($\ell_\infty$)  & 0.25 ($\ell_2$)  \\ \hline\hline
Natural dataset             & 0.01/70.94     &  1.6/70.94         \\ 
Adv. data of natural classifier & 7.35/61.66           & 25.06/65.29   \\ 
% Adv. data of robust classifier  &                     &                     &               &              \\ \hline
Robust dataset (ours)                       & \textbf{15.71}/61.60                & \textbf{34.72}/62.91   \\ \hline
\end{tabular}
%}
\end{table}

\end{document}